\documentclass[mnsc]{informs3aa}
\usepackage[final]{showkeys}
\usepackage{subfigure}
\OneAndAHalfSpacedXI 

\usepackage{endnotes}

\usepackage{longtable}
\usepackage{tabularx}
\usepackage{booktabs}
\usepackage{etex}
\usepackage{multirow}
\usepackage{array}
\usepackage{bbm}

\usepackage[title]{appendix}

 \usepackage{algpseudocode}

\newcommand{\vct}{\boldsymbol }

\newcommand{\ud}{\mathrm{d}}

\newcommand{\Lap}{\mathrm{Lap}}

\newcommand{\tv}{\mathsf{TV}}

\newcommand{\sR}{\mathfrak R}

\renewcommand{\hat}{\widehat}
\renewcommand{\tilde}{\widetilde}

\usepackage{amsbsy, amsfonts, amsgen, amsmath, amsopn, amssymb, amstext,
amsxtra, bezier, color, enumerate, graphicx, latexsym, verbatim,
pictexwd, supertabular, url, dsfont,leftidx, mathrsfs, appendix, setspace}
\usepackage{algorithm}
\usepackage{algorithmicx}
\usepackage{algpseudocode}
\makeatletter
\def\BState{\State\hskip-\ALG@thistlm}
\makeatother

\usepackage{natbib}
 \bibpunct[, ]{(}{)}{,}{a}{}{,}%
 \def\bibfont{\small}%
\usepackage[colorlinks=true,bookmarks=false,urlcolor=blue, citecolor=blue,linkcolor=blue,bookmarksopen=false,draft=false]{hyperref}

\usepackage{xcolor}
\usepackage{epstopdf}

\renewcommand{\hat}{\widehat}
\renewcommand{\tilde}{\widetilde}
\renewcommand{\bar}{\overline}

\newcommand{\const}{\mathrm{const.}}

\newcommand{\sy}{\mathsf{sy}}
\newcommand{\kl}{\mathsf{KL}}

\newcommand{\sd}{\mathfrak d}
 
\definecolor{DSgray}{cmyk}{0,1,0,0}

\TheoremsNumberedThrough     
\ECRepeatTheorems

\EquationsNumberedThrough    

\MANUSCRIPTNO{} 

\begin{document}



\RUNTITLE{Differential Privacy in Personalized Pricing}

\TITLE{Differential Privacy in Personalized Pricing with Nonparametric Demand Models}

\ARTICLEAUTHORS{%
\AUTHOR{Xi Chen\thanks{Author names listed in alphabetical order.}}
\AFF{Leonard N.~Stern School of Business, New York University, New York, NY 10012, USA\\
xc13@stern.nyu.edu}
\AUTHOR{Sentao Miao}
\AFF{Desautels Faculty of Management, McGill University, Montreal, QC H3A 1G5, Canada\\
sentao.miao@mcgill.ca}
\AUTHOR{Yining Wang}
\AFF{Warrington College of Business, University of Florida, Gainesville, FL 32611, USA\\
yining.wang@warrington.ufl.edu}
} 

\ABSTRACT{In the recent decades, the advance of information technology and abundant personal data facilitate the application of algorithmic personalized pricing. However, this leads to the growing concern of potential violation of privacy due to adversarial attack. To address the privacy issue, this paper studies a dynamic personalized pricing problem with \textit{unknown} nonparametric demand models under data privacy protection. Two concepts of data privacy, which have been widely applied in practices, are introduced: \textit{central differential privacy (CDP)} and \textit{local differential privacy (LDP)}, which is proved to be stronger than CDP in many cases. We develop two algorithms which make pricing decisions and learn the unknown demand on the fly, while satisfying the CDP and LDP gurantees respectively. In particular, for the algorithm with CDP guarantee, the regret is proved to be at most $\tilde O(T^{(d+2)/(d+4)}+\varepsilon^{-1}T^{d/(d+4)})$. Here, the parameter $T$ denotes the length of the time horizon, $d$ is the dimension of the personalized information vector, and the key parameter $\varepsilon>0$ measures the strength of privacy (smaller $\varepsilon$ indicates a stronger privacy protection). On the other hand, for the algorithm with LDP guarantee, its regret is proved to be at most $\tilde O(\varepsilon^{-2/(d+2)}T^{(d+1)/(d+2)})$, which is near-optimal as we prove a lower bound of $\Omega(\varepsilon^{-2/(d+2)}T^{(d+1)/(d+2)})$ for any algorithm with LDP guarantee. 
}



\date{}

\maketitle

\parskip=6pt

%

\section{Introduction}
From the early day bargaining to customized prices based on such as customer groups (e.g., student versus non-student), genders (e.g., personal care products, see \citealt{de2015cradle}), and regions (e.g., regional prices of AIDS drug Combivir, see \citealt{cowen2015modern}), personalized pricing has long been implemented in many commercial activities.
With the recent advance of information technology, the pricing platform could utilize customer data more efficiently, and personalized prices can be set algorithmically. For example, insurance companies quote the premium based on customers' demographic and behavioral data (see e.g., \citealt{arumugam2019survey}); hoteling websites charge different prices based on customers' locations and devices (see e.g., \citealt{vissers2014crying}). Besides industry practices, there is also a growing body of academic research on algorithmic personalized pricing (see related literature in Section~\ref{subsec:lit_review}). 

With this surge of algorithmic personalized pricing, there is growing concern of privacy issue due to potential leakage of customers' personal information. As quoted in a report by OECD Directorate For Financial And Enterprise Affairs\footnote{https://www.oecd.org/competition/personalised-pricing-in-the-digital-era.htm}, ``Similarly, the
collection and use of personal data used in personalized pricing could implicate privacy concerns.'' However, most of the practices in personalized pricing to protect data privacy are quite ad hoc such as anonymizing personal information, which cannot guarantee the data security (see e.g., \citealt{FTC2012report},  \citealt{nyt2019reidentify}). Furthermore, even if the adversary does not have direct access to the dataset, they are still able to reconstruct other customers' personal information by interacting with and observing the decisions made by the pricing platform (see e.g.,  \citealt{fredrikson2014privacy,hidano2017model}). To address these malicious attack to personal data, \citet{dwork2006our,dwork2006calibrating} proposed an important concept of the so-called \textit{differential privacy}, which is the \textit{de facto} privacy standard in practice. More specifically, there are mainly two types of differential privacy which are widely used in practice: \textit{central differential privacy (CDP)} (see e.g., \citealt{dwork2014algorithmic}) and \textit{local differential privacy (LDP)} (see e.g., \citealt{evfimievski2003limiting,duchi2013localpriv,duchi2018minimax}). { Intuitively, CDP guarantees that for any time $t$, the adversary is unlikely, depending on a privacy parameter $\varepsilon>0$ (smaller $\varepsilon$ leads to higher security; hence it is also called $\varepsilon$-CDP), to infer the data of any customer who has arrived before $t$. For example, US Census Bureau \citeyear{census2020} applied the techniques of CDP for the privacy of census data. While for LDP, the customer does not even trust the platform, so that the platform can \emph{only} use privatized historical data to make decisions. As a result, an adversary is unlikely (again, depending on $\varepsilon$, also known as $\varepsilon$-LDP) to obtain other customers' data even if it has direct access to the platform's dataset. Examples of practices of LDP include Google \citeyear{google2014} and Apple \citeyear{apple2019}. We refer to Figure \ref{fig:cdp_ldp} for a graphical representation of CDP and LDP. 
As shown in the left panel of Figure \ref{fig:cdp_ldp}, for CDP, the \emph{trusted} aggregator (i.e., the platform) collects historical data as well as the query from the arriving customer, and it outputs a \emph{privatized} answer (e.g., price) such that an adversary cannot infer sensitive information from this answer. The LDP is illustrated 
in the right panel of Figure \ref{fig:cdp_ldp}, where the aggregator (e.g., the platform) is \emph{untrusted} so that it can only collect privatized/perturbed data from customers. }
Please refer to Section~\ref{sec:dp_notions} for the detailed formulation and comparisons between CDP and LDP in our setting.

\begin{figure}[htbp]
	\centering
	\includegraphics[width=1\textwidth]{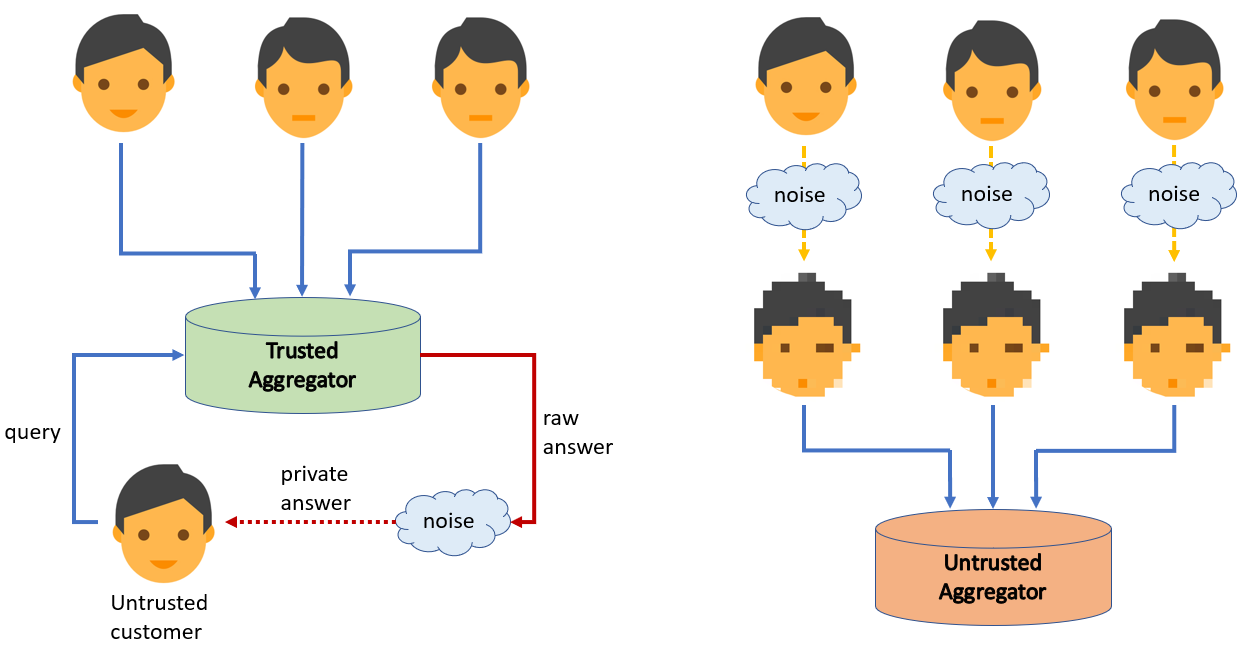}
	\caption{CDP (the left panel) and LDP (the right panel) in a general setting.} \label{fig:cdp_ldp}
\end{figure}

In this paper, we address the concern of protecting customers' data in a dynamic personalized pricing problem with demand learning. Briefly, there is a finite selling horizon with length $T$, and in each time period $t$, there is one customer with a $d$-dimensional feature/data vector $\vct x_t\in\mathbb{R}^d$ arriving at the platform for purchasing a single product. To maximize the cumulative revenue, the platform needs to decide a personalized price $p_t$ based on the knowledge of the unknown demand model (from historical data) and $\vct x_t$, while at the same time protecting customer's data.
Assuming the demand model to be \emph{nonparametric}, we propose two algorithms which protect data privacy by satisfying $\varepsilon$-CDP and $\varepsilon$-LDP respectively. The main contributions of this paper is summarized as follows.

\paragraph{\bf Protecting data privacy with nonparametric demand model.} As mentioned, the demand model in this paper is assumed to be nonparametric (see \citealt{chen2021nonparametric} for preceding work without data privacy), as opposed to the existing literature on data privacy in pricing which assumes parametric demand (see e.g., \citealt{lei2020privacy,chen2020privacy,Han:21:generalized}). 
The preservation of data privacy in such nonparametric settings gives rises to several technical challenges, which we describe in more details below.
\begin{itemize}
	\item In the work of \cite{chen2021nonparametric}, the authors divide the space of contextual vectors into local \emph{hypercubes} (i.e., each arriving customer belongs to a certain hypercube depending on his or her data $\vct x_t$), and	a pricing algorithm is proposed that runs some multi-armed bandit algorithm for each {hypercube}. 
	While this method works nicely when data privacy is not of concern, it is difficult to adapt it under either CDP or LDP settings. The main reason is that, to protect the privacy of the personalized price offered to a customer, one needs to anonymize the information on \emph{all} arms during a certain time range.
	Since the number of arms in the algorithm of \cite{chen2021nonparametric} must scale polynomially with $T$,
	this simple approach will lead to an inflated and sub-optimal regret.
	\item Many existing works on pricing with nonparametric demands (either contextual or not) utilize the \emph{trisection search}
	approach to search for the revenue-optimal prices by exploiting the concavity structures of the revenue function (with respect to price) without other prior knowledge 
	\citep{wang2014close,lei2014near}.
	While the trisection search approach has the advantage of maintaining a constant number of cumulative revenue statistics
	(in contrast to the above-mentioned multi-armed bandit approach), there is another subtle technical difficulty of directly applying it
	in the privacy preserving setting: when working with trisection search, we need a pre-allocated sample budget for each iteration
	to decide when to stop comparing the two mid-points.
	With data privacy (especially local privacy) constraints, it is difficult for the algorithm to maintain the number of samples or customers
	already arrived who belong to a certain hypercube. Therefore, the algorithm has difficulty knowing when to stop comparing mid-points in a trisection search procedure.
\end{itemize}

To address the above-mentioned technical challenges, in this paper we use the idea of \emph{quadrisection search} within each hypercube
of customers' contextual vectors.
In quadrisection search, the price interval is divided equally into four pieces with three mid-points. This gives the algorithm more information
to decide the direction of price interval shrinkage, without the need to maintain an accurate counting of customers arriving in each hypercube during a certain time range.
It also avoids the regret inflation problem from multi-armed bandit approaches because at each time the quadrisection search method
maintains only 5 (or 10) anonymized statistics for each hypercube, which will not add unreasonable cost due to privacy preservation.


\paragraph{\bf Near-optimal pricing algorithms preserving data privacy.} 
We present two algorithms named \emph{CPPQ} and \emph{LPPQ} for nonparametric personalized pricing with privacy guarantees. {Both algorithms satisfy $\varepsilon$-CDP (for \emph{CPPQ}) or $\varepsilon$-LDP (for \emph{LPPQ}) constraints
regardless of choices of algorithm input parameters (except $\varepsilon$); thus in practice, we can tune the input parameters for better performance without worrying about privacy preservation.}
In addition to the local quadrisection search method mentioned in the previous paragraph, the proposed privacy-preserving algorithms
also employ several advanced techniques such as tree-based aggregation and noisy statistical counts to ensure customers' data privacy,
which we discuss and explain in more details later in the paper when we describe the proposed algorithms.

In addition to rigorous privacy guarantees, we also demonstrate that when certain algorithm parameters are carefully chosen,
the proposed algorithms have near-optimal regrets (upto logarithmic factors in $T$).
More specifically, for the CDP setting the proposed algorithm enjoys a regret upper bound of $\widetilde O(T^{(d+2)/(d+4)}+\varepsilon^{-1}T^{d/(d+4)})$,
which is optimal when $\varepsilon$ is not too small because $\Omega(T^{(d+2)/(d+4)})$ is a known \emph{lower bound} for personalized pricing
with nonparametric demands even when data privacy is not of concern \citep{chen2021nonparametric}. We note that the notation $\widetilde O(\cdot)$ hides some logarithmic factors in $T$ (see Theorem \ref{thm:cppt-regret} for more details).
For the more challenging LDP setting\footnote{Proposition \ref{prop:lpd-adp} shows that LDP is a stronger data privacy notion compared with CDP under certain assumptions.}, our proposed LPPQ algorithm achieves a regret upper bound of $\widetilde O(\varepsilon^{-2/(d+2)}T^{(d+1)/(d+2)})$ (see Theorem \ref{thm:regret-lppt}).
While this regret upper bound is considerably worse than the $T^{(d+2)/(d+4)}$ scaling even if $\varepsilon=\Omega(1)$,
we show that it is indeed \emph{rate-optimal}, as explained in the next paragraph.

\paragraph{\bf Minimax lower bound for locally private personalized pricing.} In this paper we prove a minimax lower bound of $\Omega(\varepsilon^{-2/(d+2)}T^{(d+1)/(d+2)})$ on the regret of \emph{any} possible personalized pricing policy for $T$ sequentially arriving customers subject to the $\varepsilon$-LDP constraint (see Theorem \ref{thm:regret-lower-bound}).
While minimax lower bounds of locally private estimators have been previously studied \citep{duchi2018minimax},
the lower bound in our problem setting is more complicated because the prices offered to sequentially arriving customers are \emph{adaptive}
and not independently distributed with respect to any underlying distribution, which is the case in the work of \cite{duchi2018minimax}.
To establish a lower bound for \emph{any adaptive} personalized pricing strategy subject to local privacy constraints, 
we carefully generalize the information theoretical arguments in \cite{duchi2018minimax} to adaptively collected data
and obtain a tight minimax lower bound.
More details and discussion is presented in Section~\ref{sec:lower-bound}.

\subsection{Related Literature}\label{subsec:lit_review}


This section reviews some related literature from two streams: the theory and application of data privacy, and the related paper in personalized pricing with demand learning. 

\textbf{Literature in data privacy.}  The concept of data privacy was first rigorously quantified by an important framework called differential privacy (DP), which was first introduced in \citet{dwork2006calibrating,dwork2006our}. The definition of DP has become a \textit{de facto} standard for data privacy in both academic and industrial practice. We refer the interested readers to \citet{dwork2014algorithmic,acquisti2016economics} for comprehensive reviews. Based on this concept, different theories and techniques have been developed, such as the technique of tree-based aggregation \citep{chan2011private}, which is used in the design of our algorithm CPPQ, and the mechanism of random perturbation (see e.g., \citealt{dwork2014algorithmic}), which is an important component in the design of both CPPQ and LPPQ.

A more related stream of literature to ours is the so-called online learning with differential privacy. That is, the decision maker has to learn the environment and make decisions on the fly, while preserving the data privacy. For example, \citet{mishra2015nearly} proposed (central) differentially private UCB and Thompson Sampling algorithms for multi-armed bandit (MAB) problems. 
\citet{shariff2018differentially}  studied linear contextual bandit under the CDP constraint, and \citet{ren2020multi,zheng2020locally} studied the (contextual) bandit problem with $\varepsilon$-LDP guarantee. { Later, \citet{Han:21:generalized} extended the results to generalized linear bandits with stochastic contexts. In particular, the authors leveraged the idea of stochastic gradient descent and proposed a novel LDP strategy so that their algorithms are proved to have regret $O((\ln(T)/\varepsilon)^2)$ or $\tilde O(T^{(1-\beta)/2}/\varepsilon^{1+\beta})$, depending on whether some ``optimality margin'' parameter $\beta=1$ or $\beta\in(0,1)$. }
Our work differ from the private MAB literature in that our model is nonparametric as opposed to the (generalized) linear model in MAB literature. Therefore, we cannot preserve the privacy (either centrally or locally) through protecting the demand parameters as in parametric models. Besides MAB problems, there are some other differentially private online learning problems such as private sequential learning \citep{xu2018query,Tsitsiklis2020private,xu2020optimal}, and dynamic pricing \citep{chen2020privacy}, which will be discussed later in literature review in personalized pricing. 

{ We also note that there is a growing body of literature on data privacy in service systems. For instance, \citet{Hu:20:privacy} studied customer-centric privacy management under queueing models, where customers are strategic in deciding whether to disclose private personal information to the service provider.}

\textbf{Literature in personalized pricing with demand learning.} As we discussed in the introduction, algorithmic dynamic pricing with demand learning has been increasingly popular especially in recent years (for an incomplete list of literature, see, e.g., \citealt{araman2009dynamic,besbes2009dynamic,farias2010dynamic,harrison2012bayesian,broder2012dynamic,den2013simultaneously,wang2014close,chen2015real,besbes2015suprising,cheung2017dynamic,ferreira2018online,wang2021uncertainty,miao2021general}). With the abundance of customer's personal data, there is also a growing trend of implementing personalized prices based on customers' contextual information. For instance, \citet{qiang2016dynamic} applied a greedy iterated least squares method on linear demand function. \citet{ban2021personalized} and \citet{javanmard2019dynamic} studied the high dimensional personalized pricing with parametric demand model and sparse parameters. \citet{keskin2020data} leveraged data clustering for customized electricity pricing. 

With widespread public concern of personal data security, several recent works in personalized pricing start to take the data privacy into consideration (see e.g., \citealt{tang20contextual,lei2020privacy,chen2020privacy,bimpikis2021data}). Among these literature, the most related work to this paper is \citet{chen2020privacy}, which also consider a dynamic personalized pricing problem with differential privacy guarantee. Compared with this paper, our work has the following differences. First, \citet{chen2020privacy} studied a parametric demand model, while this work considers a nonparametric model, making the demand learning and privacy protection completely different. Second, the differential privacy studied in \citet{chen2020privacy} is the $\varepsilon$-CDP, while our paper not only proposes an algorithm for $\varepsilon$-CDP, but also develops an algorithm for $\varepsilon$-LDP --- a stronger notation of privacy than $\varepsilon$-CDP in many cases. Another related paper to ours is \citet{lei2020privacy}, who also studied pricing algorithms satisfying $\varepsilon$-CDP and $\varepsilon$-LDP. The difference in their work is that their setting is an offline pricing problem while ours is online dynamic pricing with demand learning, and the demand model considered in \citet{lei2020privacy} is parametric (in contrast to the nonparametric model in this paper).  
To the best of our knowledge, this paper is the first to consider a dynamic personalized pricing problem with nonparametric demand under both CDP and LDP guarantees. 

\subsection{Paper Organization}
The rest of this paper is organized as follows. In Section \ref{sec:model}, the model and some technical assumptions are introduced. Section \ref{sec:dp_notions} introduces the important definitions of $\varepsilon$-CDP and $\varepsilon$-LDP in our problem setting, and their relationship. After that, the algorithms satisfying $\varepsilon$-CDP and $\varepsilon$-LDP are introduced in Section \ref{sec:cdp} and Section \ref{sec:ldp} respectively. Furthermore, we develop a lower bound for any algorithm with $\varepsilon$-LDP in Section \ref{sec:lower-bound}. In Section \ref{sec:num_exp}, some numerical experiments are used to illustrate the performance of the proposed algorithms. In the end, the paper is concluded in Section \ref{sec:conclusion}. Some technical proofs can be found in the appendix. 

\section{The Model and Assumptions}\label{sec:model}

We study a stylized dynamic personalized pricing problem of a single type of product for $T$ consecutive selling periods.
At the beginning of selling period $t$, the pricing platform observes a contextual vector $\vct x_t\in\mathcal X\subseteq\mathbb R^d$ of the incoming customer.
The platform then offers a price $p_t\in[\underline p,\overline p]$, and the stochastically realized demand $y_t\in\mathcal Y\subseteq\mathbb R^+$ is being modeled 
by an unknown nonparametric model $\lambda:[\underline p,\overline p]\times\mathcal X\to\mathbb R^+$ as
\begin{equation}
\mathbb E[y_t|\vct x_t,p_t] = \lambda(p_t,\vct x), \;\;\;\;\;\; y_t\in\mathcal Y.
\label{eq:demand-model}
\end{equation}
Throughout the paper, we also define $f(p,\vct x):=p\lambda (p,\vct x)$ as the function that gives the expected revenue of price $p$
conditioned on customer context $\vct x$.

Clearly, when $\lambda$ (and subsequently $f$) is known \textit{a priori} to the platform, the optimal pricing strategy (without considerations of privacy ocncerns)
would be to simply set $p_t=p^*(\vct x_t)=\arg\max_{p\in[\underline p,\overline p]}f(p_t,\vct x_t)$.
Without knowing $\lambda$ or $f$, on the other hand, requires the platform to learn the unknown demand model and offer near-optimal personalized prices
simultaneously, commonly known in the literature as the \emph{exploration-exploitation tradeoff}.
We adopt the classical measure of \emph{cumulative regret} (we also call it \textit{regret} for brevity) to measure the performance of a pricing policy $\pi$ over $T$ time periods.
(See the next section for a rigorous definition of an admissible pricing policy and when it satisfies privacy guarantees.)
More specifically, the regret of a policy $\pi$ under model $f$ is defined as
\begin{equation}
\sR_T(f,\pi) := \mathbb E\left[\sum_{t=1}^T (f(p^*(\vct x_t),\vct x_t)-f(p_t,\vct x_t)
)\right],\;\;\;\;\;\;\text{where}\;\;p^*(\vct x_t)=\arg\max_{p\in[\underline p,\overline p]}f(p,\vct x_t).
\label{eq:defn-regret}
\end{equation}

We make the following assumptions throughout this paper:
\begin{enumerate}
\item[(A1)] The domains of $(x_t,y_t,p_t)$ satisfy $\mathcal X\subseteq[0,1]^d$, $\mathcal Y\subseteq[0,1]$ and $[\underline p,\overline p]\in[0,1]$.
Furthermore, 
$\vct x_t$ are i.i.d. over $t\in[T]$ which follows an unknown underlying distribution $P_X$ supported on $\mathcal X$ with probability density function $\chi:\mathcal X\to\mathbb R^+$
that satisfies $\sup_{\vct x\in\mathcal X}\chi(\vct x)\leq C_X$ almost surely;
\item[(A2)] There exists a finite constant $C_L<\infty$ such that $|f(p,\vct x)-f(p',\vct x')|\leq C_L(|p-p'|+\|\vct x-\vct x'\|_2)$ for all $p,p'\in[\underline p,\overline p]$
and $\vct x,\vct x'\in\mathcal X$;
\item[(A3)] For any hypercube $B\subseteq\mathcal X$ and $p\in[\underline p,\overline p]$, define the expected revenue and the optimal price on the hypercube $B$,
\begin{equation}\label{eq:opt_cube}
f_B(p) := \mathbb E_{P_X}[f(p,\vct x)|\vct x\in B], \qquad 
p^*(B)=\arg\max_{p\in[\underline p,\overline p]}f_B(p).
\end{equation}
Then there exist uniform constants $0<\sigma_H\leq C_H<\infty$ and $C_p<\infty$ such that
\begin{enumerate}
\item[(A3-a)] $p^*(B)\in(\underline p,\overline p)$; furthermore, $f_B(\cdot)$ is twice continuously differentiable in $p$ and satisfies
$\sigma_H^2\leq -f_B''(p)\leq C_H^2$ for all $p\in(\underline p,\overline p)$;
\item[(A3-b)] $\inf_{\vct x\in B}p^*(\vct x)\leq p^*(B)\leq \sup_{\vct x\in B}p^*(\vct x)$;
\item[(A3-c)] $\sup_{\vct x\in B}p^*(\vct x)-\inf_{\vct x\in B}p^*(\vct x)\leq C_p\sup_{\vct x,\vct x'\in B}\|\vct x-\vct x'\|_2$.
\end{enumerate}
\end{enumerate}

{Assumptions (A1) and (A2) are standard regularity assumptions in the literature (see, e.g.,\citealt{qiang2016dynamic,javanmard2019dynamic,ban2021personalized} for parametric demand models satisfying these two assumptions).
Assumption (A3) follows the model setup in the work of \cite{chen2021nonparametric}. First of all, the intuition of $f_B(p)$ is the average revenue of price $p$ when the context $\vct x$ is in the hypercube $B$. One can think of a case where the retailer can only observe $\vct{1}\{\vct x\in B\}$ instead of the value of $\vct x$, so that the best pricing strategy is obviously to charge $p^*(B)$. For the three technical conditions of Assumption (A3), (A3-b) and (A3-c) are exactly the same as in \citet{chen2021nonparametric} (in particular, the part two and three of Assumption 3). The only slightly different assumption compared with (part one of) Assumption 3 in \citet{chen2021nonparametric} (which basically assumes that $|f_B(p^*(B))-f_B(p)|=\Omega((p^*(B)-p)^2)$) is our (A3-a). 
On one hand,  as shown in Proposition 1 of \citet{chen2021nonparametric}, our assumption (A3-a) implies its counterpart in the first part of Assumption 3 in \citet{chen2021nonparametric}. On the other hand, for all the motivating examples studied in Remark 1 in \citet{chen2021nonparametric} (e.g., the linear covariate case, separable demand functions, and localized functions), 
our Assumption (A3-a) is satisfied as well.
}


\section{Central and Local Differential Privacy}\label{sec:dp_notions}
This section introduces two important concepts of differential privacy: the CDP and LDP. In the following two subsections, we will discuss each of these two concepts adapted to our problem setting, and illustrate their relationship, especially their differences. 

\subsection{Central Differential Privacy}

We first introduce the standaerd definition of (central) differential privacy \citep{dwork2014algorithmic,dwork2006calibrating,dwork2006our}
for offline problems. 
\begin{definition}[differential priavcy]
Let $s_t:=(\vct x_t,y_t,p_t)\in\mathcal S:=\mathcal X\times\mathcal Y\times[\underline p,\overline p]$.
Let $\pi:(s_1,\cdots,s_T)\mapsto o$ be an offline randomized algorithm that takes customers' sensitive information as input and outputs statistics or decision $o\in\mathcal O$.
For any $\varepsilon>0$, $\pi$ satisfies $\varepsilon$-differential privacy if for all $s_1,\cdots,s_t,\cdots,s_T$ and $s_t'\neq s_t$,
and any measurable $O\subseteq\mathcal O$, 
it holds that
$$
\Pr[\pi(s_1,\cdots,s_t,\cdots, s_T)\in O] \leq e^{\varepsilon}\cdot \Pr[\pi(s_1,\cdots,s_t',\cdots,s_T)\in O].
$$
\label{defn:classical-dp}
\end{definition}

In Definition \ref{defn:classical-dp}, the output domain $\mathcal O$ captures all information released by the randomized algorithm $\pi$
to the general public, which would be the offered prices $\{p_t\}_{t=1}^T$ in our setting.
Such a definition, however, poses technical challenges in the context of dynamic personalized pricing as the price $p_t$ for customer arriving at time $t$
is highly indicative of the customer's personal information (see e.g., the discussion in Section 4.2 in \citealt{chen2020privacy}). More specifically, Proposition 1 in \citet{chen2020privacy} proves that any pricing policy satisfying Definition \ref{defn:classical-dp} will have worst-case regret at least $\Omega(T)$, suggesting that this definition is too strong for our setting. As a result, in the works of \citep{shariff2018differentially,chen2020privacy}, the notion of \emph{anticipating differential privacy} (for brevity of notation, we just call it central differential privacy, or CDP)
is introduced to focus on the impact of sensitive customers' information on \emph{future} prices as formalized below.
\footnote{The anticipating privacy notion defined here is slightly weaker than the definitions in \citep{shariff2018differentially,chen2020privacy},
which considered joint distributions of future prices. Nevertheless, we adopt this definition here to be more compatible
with existing definitions of local privacy notations, which is also appropriate for practical usage.}
\begin{definition}[central differential privacy for personalized pricing]
Let data of customer $t$ be $s_t:=(\vct x_t,y_t,p_t)\in\mathcal S:=\mathcal X\times\mathcal Y\times[\underline p,\overline p]$.
Let $\pi=(A_1,A_2,\cdots,A_T)$ be an admissible personalized pricing policy, where $A_t(\cdot|s_1,\cdots,s_{t})$ is a distribution of $p_t\in[\underline p,\overline p]$
measurable conditioned on $\{s_1,\cdots,s_t\}$.
We say $\pi$ satisfies $\varepsilon$-central differential privacy ($\varepsilon$-CDP)
if for all $j<t$, $s_1,\cdots,s_j,\cdots,s_{t-1}$ and $s_j'\neq s_j$ with $y_t,y_t'\in\mathcal Y$, and measurable $U\subseteq[\underline p,\overline p]$, it holds that
$$
A_t(U|\vct x_t,s_1,\cdots,s_j,\cdots,s_{t-1}) \leq e^{\varepsilon}A_t(U|\vct x_t,s_1,\cdots,s_j',\cdots,s_{t-1}).
$$
\label{defn:anticipating-dp}
\end{definition}

Intuitively, Definition \ref{defn:anticipating-dp} requires the pricing policy $\pi$ to be stable with respect to the sensitive information of a customer
\emph{prior to} the current time period, so that a malicious third party arriving at time $t$ cannot reliably infer protected information of 
previously arrived customers.
One key difference between Definitions \ref{defn:anticipating-dp} and \ref{defn:classical-dp} is the notable exclusion of $s_t$ from the conditional set (i.e., we do not require the pricing policy to be stable with respect to $\vct x_t$),
meaning that we assume the customer arriving at time $t$ is either 1) a malicious third-party who is therefore not interested in its own protected information,
or 2) a customer trusting that her sensitive information will \emph{not} be released to later customers directly or indirectly through prices. We refer to Figure \ref{fig:cdp} as a graphic illustration of $\varepsilon$-CDP.

\begin{figure}
	\centering
	\includegraphics[scale=0.33]{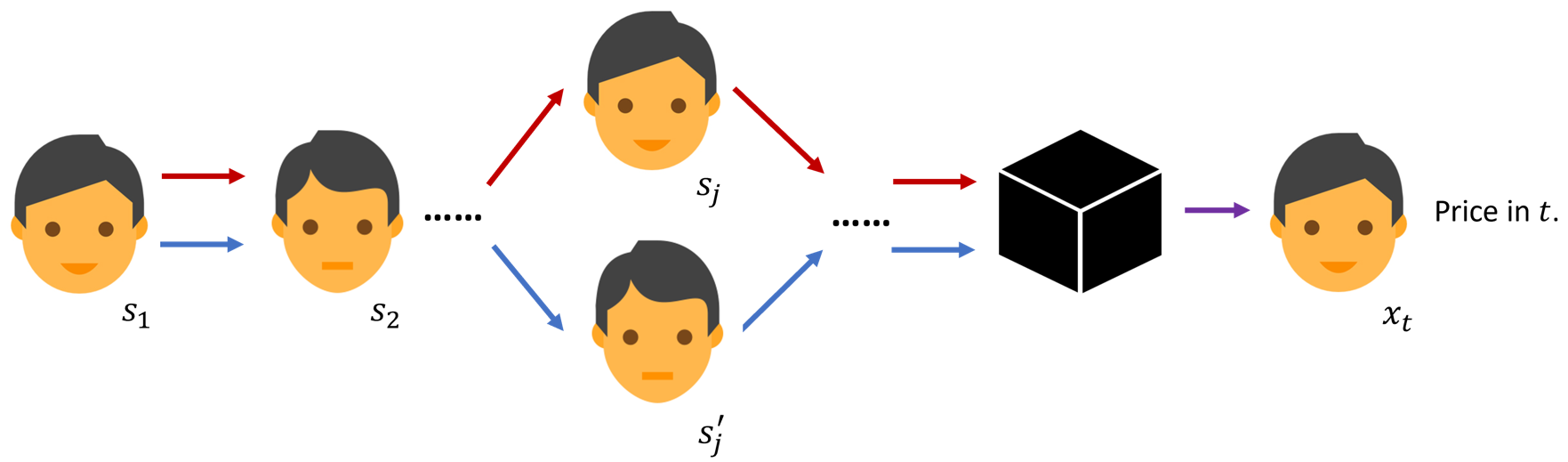}
	\caption{Graphic representation of $\varepsilon$-CDP. The red and blue arrows mean that two neighboring sequences of data (with the only difference in $s_j$ and $s_j'$) are not likely to be distinguished after privatization.}
	\label{fig:cdp}
\end{figure}

\subsection{Local Differential Privacy}

\begin{figure}
		\centering
		\includegraphics[scale=0.27]{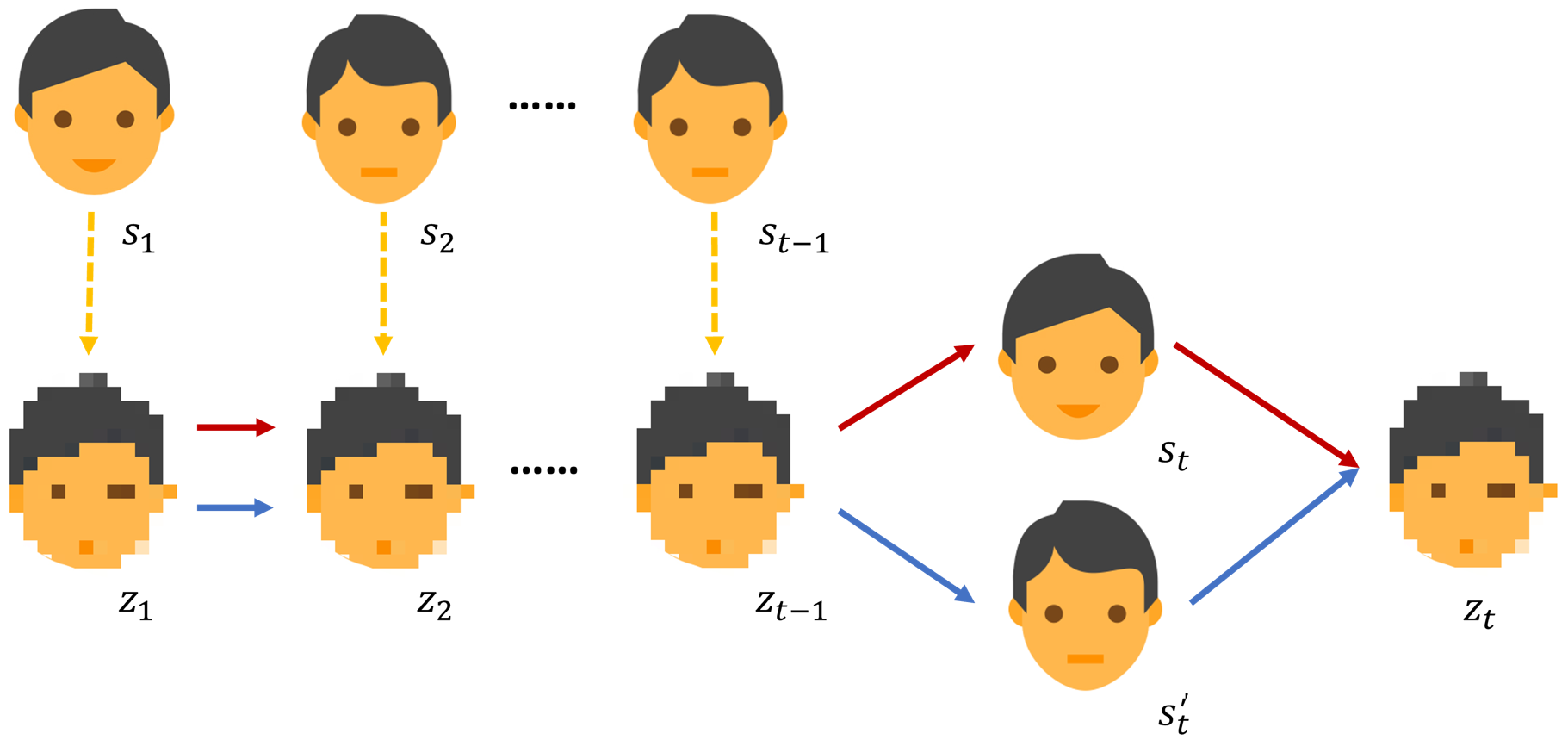}
		\caption{Graphic representation of statistics recorder $Q_t$ of $\varepsilon$-LDP. The red and blue arrows mean that privatizing different $s_t$ and $s_t'$ with the sequence of historical privatized data is not likely to be distinguished.}
		\label{fig:ldp_stat_recorder}
\end{figure}

\begin{figure}
	\centering
	\includegraphics[scale=0.25]{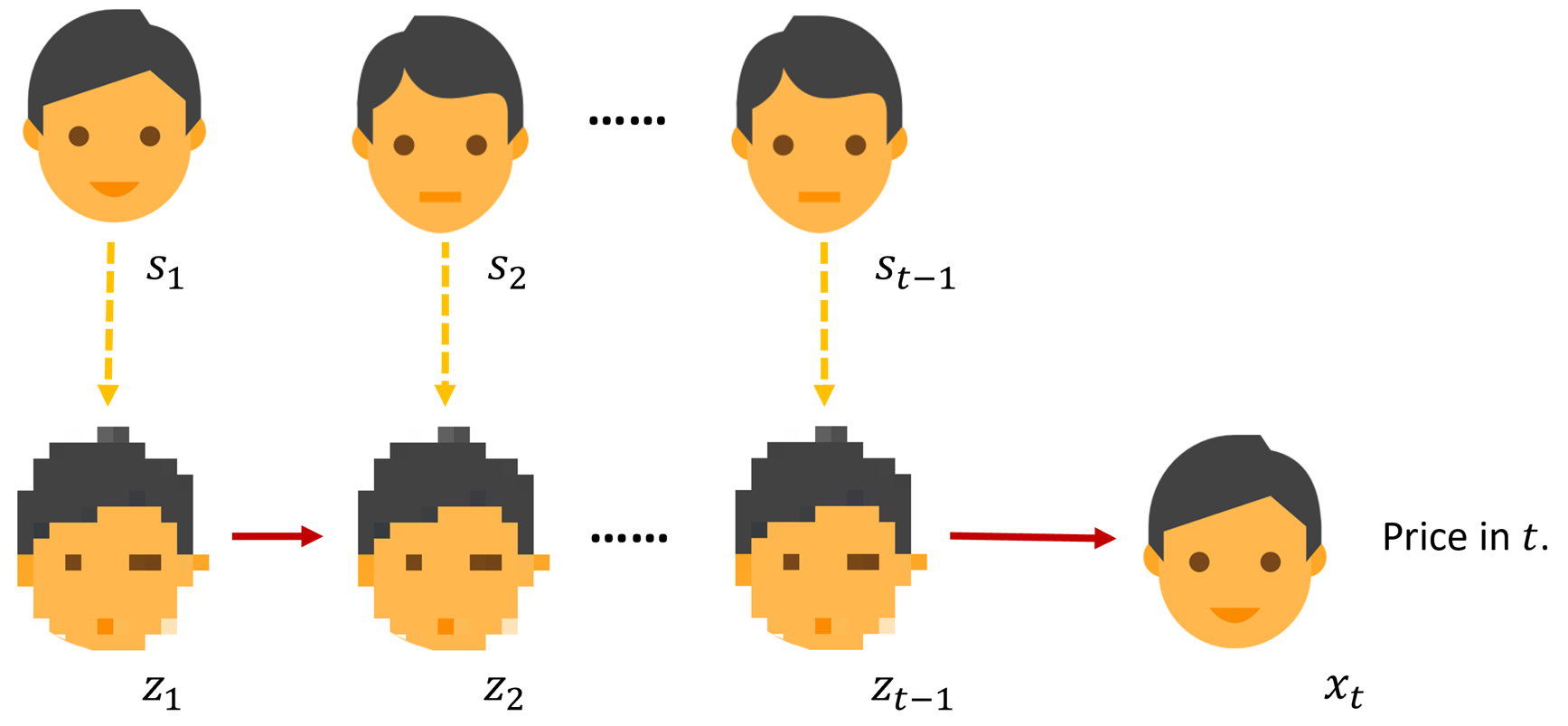}
	\caption{Graphic representation of pricing strategy $A_t$ of $\varepsilon$-LDP. This graph shows that the platform only uses the privatized historical data (instead of the real historical data) for pricing.}
	\label{fig:ldp_pricing}
\end{figure}

One fundamental assumption being made in the central differential privacy notions (including both Definitions \ref{defn:classical-dp} and \ref{defn:anticipating-dp}) is that the customers trust the \emph{platform} to protect their sensitive information,
and are only worried about malicious third parties pretending to be customers illegally extracting data indirectly through offered prices.
{ This is clear from the fact that the pricing policy $A_t$ is conditioned on the true data $\{s_{j}=(\vct x_j,y_j,p_j):j<t\}$ for all customers arriving before $t$. That is, the platform can use the actual data of its customers to make pricing decision in each time period.}
On the other hand, in local differential privacy formulations, it is essential to constrain the platform (i.e., its pricing algorithm/policy) so that the platform never stores the actual sensitive data of customers and instead make pricing decisions based on anonymized information.
More specifically, let $z_t\in\mathcal Z$ be the privatized/anonymized statistics the platform records at time $t$,
and abbreviate $\vct z_{<t}=(z_1,\cdots,z_{t-1})$ as the anonymized statistics of all customers arriving prior to time $t$.
The firm's (locally private) dynamic personalized pricing strategy can be parameterized by two (sequences of) conditional distributions:
the \emph{statistics recorder} $Q_t$, and the pricing strategy $A_t$, as follows:
\begin{eqnarray}
z_t &\sim& Q_t(\cdot|s_t, z_1,\cdots, z_{t-1}); \label{eq:defn-Qt}\\
p_t &\sim& A_t(\cdot|\vct x_t, z_1,\cdots,z_{t-1}).\label{eq:defn-At}
\end{eqnarray}
For notational convenience, in the rest of the paper we write $z_t=Q_t(s_t,\vct z_{<t})$ and $p_t=A_t(\vct x_t,\vct z_{<t})$ with the understanding
that both $Q_t$ and $A_t$ are randomized functions/procedures.
In non-private settings, one can simply let $z_t=s_t$ and then $p_t=A_t(\vct x_t,s_1,\cdots,s_{t-1})$ reduces to the standard personalized pricing policy.
With (local) privacy constraints, however, the statistics $\{z_t\}$ being recorded are constrained to be privatized statistics (i.e., do not leak
too much protected information $\{s_t\}$) and the pricing decisions $p_t$ are forced to be made upon the anonymized statistics $\{z_t\}$
instead of the sensitive data $\{s_t\}$.
It is also clear from Eq.~(\ref{eq:defn-Qt}) that $z_t$ and $s_j$ are independent conditioned on $s_t,\vct z_{<t}$, for all $j<t$. We refer to Figure \ref{fig:ldp_stat_recorder} and Figure \ref{fig:ldp_pricing} for graphic representations of $Q_t$ and $A_t$ respectively. 


%

The formal definition of local differential privacy, following the seminal works of \cite{evfimievski2003limiting,duchi2018minimax} in the literature, is given below:
\begin{definition}[local differential privacy for personalized pricing]
For any $\varepsilon>0$, a personalized pricing policy $\pi=\{Q_t,A_t\}_{t=1}^T$ satisfies $\varepsilon$-local differential privacy ($\varepsilon$-LDP)
if for every $t$, $\vct z_{<t}$ and $s_t\neq s_t'$, it holds for every measurable $Z_t\subseteq\mathcal Z$ that
$$
Q_t(Z_t|s_t,\vct z_{<t}) \leq e^{\varepsilon}\cdot Q_t(Z_t|s_t',\vct z_{<t}).
$$
\label{defn:ldp}
\end{definition}

{ We provide some additional notes for the definition of LDP. First, LDP guarantees the privacy of customer's data by the statistics recorder $Q_t$ instead of directly by $A_t$ as in CDP. More specifically, by privatizing $s_t$ through $Q_t$, the pricing policy $A_t$ can only use privatized data for decision making. As a result, even if the adversary in time $t$ is able to infer any $z_j$ where $j<t$ or hack the whole dataset of the platform,  the $j$-th customer's personal data is still protected as none of the true data $s_j$ is stored in the system. Second, in LDP the platform also makes pricing decision $A_t$ conditioned on customer's raw data $\vct x_t$ as in CDP. Again this can happen when customer in $t$ is the adversary, who  will not  hack his/her own data. In the case of a normal customer, statistics recorder $Q_t$ guarantees that $\vct x_t$ is unlikely to be leaked to others; thus customer $t$ can still trust the platform to use his/her personal data. 	
}

While local and central differential privacy are not necessarily comparable, in the special case of the ``non-interactive regime''
where $\{Q_t\}$ are independent distributions, the following proposition shows that local differential privacy is stronger than its centralized counterpart.
\begin{proposition}
Let $\pi=\{Q_t,A_t\}_{t=1}^T$ be a personalized pricing policy satisfying $\varepsilon$-LDP.
Suppose also that $Q_t(\cdot)$ is independent of $\vct z_{<t}$ for all $t$.
Then $\pi$ satisfies $\varepsilon$-CDP.
\label{prop:lpd-adp}
\end{proposition}
\begin{proof}{Proof of Proposition \ref{prop:lpd-adp}.}
	This is true because
	\[
	\begin{split}
	&A_t(U|\vct x_t,s_1,\ldots,s_j,\ldots,s_{t-1})\\
	=&A_t(U|\vct x_t,z_1,\ldots,z_j,\ldots,z_{t-1})\mathbb{P}(z_1,\ldots,z_j,\ldots,z_{t-1}|s_1,\ldots,s_j,\ldots,s_{t-1})\\
	=&A_t(U|\vct x_t,z_1,\ldots,z_j,\ldots,z_{t-1})\prod_{s\in[t]\backslash\{j\}}Q_s(z_s|s_s)Q_j(z_j|s_j)\\
	\leq& e^{\varepsilon}A_t(U|\vct x_t,z_1,\ldots,z_j,\ldots,z_{t-1})\prod_{s\in[t]\backslash\{j\}}Q_s(z_s|s_s)Q_j(z_j|s_j')\\
	=&A_t(U|\vct x_t,s_1,\ldots,s_j',\ldots,s_{t-1})
	\end{split}
	\]
	where the second equality is because $Q_t(\cdot)$ is independent of $\vct z_{<t}$ for all $t$, and the inequality is from the definition of LDP, which states that $Q_j(z_j|s_j)\leq e^\epsilon Q_j(z_j|s_j')$ for any $j<t$ and $s_j,s_j'\in\mathcal{S}$. $\square$
	
\end{proof}

\section{The \textsc{Centralized-Private-Parallel-Quadrisection} (CPPQ) Algorithm}\label{sec:cdp}

In this section we describe a personalized pricing algorithm that satisfies the $\varepsilon$-CDP as defined in Definition \ref{defn:anticipating-dp}.
The proposed algorithm is named \textsc{Centralized-Private-Parallel-Quadrisection} (CPPQ) and its pseudocode description is given in Algorithm \ref{alg:cppt}. There are several important techniques in CPPQ which will be used as building blocks for algorithms with $\varepsilon$-LDP (see Section \ref{sec:ldp}).

\begin{algorithm}[t]
\caption{The \textsc{Centralized-Private-Parallel-Quadrisection} (CPPQ) algorithm}
\label{alg:cppt}
\begin{algorithmic}[1]
\State \textbf{Input}: time horizon $T$, privacy parameter $\varepsilon$, algorithm parameters $J,c_1,c_1',c_2$.
\State Initialize: partition $[0,1]^d$ into $J$ equally-sized hypercubes (each side's length being $h=J^{-1/d}$);
for each hypercube $B_j$, $j\in[J]$, let $\vct\rho_j=(\underline p,\frac{1}{4}\underline p+\frac{3}{4}\overline p, \frac{1}{2}\underline p+\frac{1}{2}\overline p,\frac{3}{4}\underline p+\frac{1}{4}\overline p,\overline p)\in[\underline p,\overline p]^5$; $r_{j,k}(0)=\mu_{j,k}(0)=0$ for $k\in\{1,2,3,4,5\}$;
$\varsigma_j\gets 0$;
\For{$t=1,2,\cdots,T$}
	\State Receive $\vct x_t\in[0,1]^d$ and let $j_t\in[J]$ be the index such that $\vct x_t\in B_{j_t}$;
	\State Offer price $p_t=\rho_{j_t,k_t}$ where $k_t\equiv t\mod 5$, and receive $y_t\in[0,1]$;
	\For{$j=1,2,\cdots,J$}
		\State Let $u_{t,j,k}=\vct 1\{j_t=j\wedge k_t=k\}y_tp_t$ and $v_{t,j,k}=\vct 1\{j_t=j\wedge k_t=k\}$ for $j\in[J]$ and $k\in[5]$;
		\State Update $r_{j,k_t}(t)\gets\textsc{TreeBasedAggregation}(\{u_{\tau,j,k_t}\}_{\tau<t},t,u_{t,j,k_t},\varepsilon/2,T)$;
		\State Update $\mu_{j,k_t}(t)\gets\textsc{TreeBasedAggregation}(\{v_{\tau,j,k_t}\}_{\tau<t},t,v_{t,j,k_t},\varepsilon/2,T)$;
		\State Let $r_{j,k}(t)\gets r_{j,k}(t-1)$ and $\mu_{j,k}(t)\gets \mu_{j,k}(t-1)$ for $k\neq k_t$;
		\State For $k\in[5]$ compute $\hat r_{jk}=r_{j,k}(t)-r_{j,k}(\varsigma_j)$ and $\hat\mu_{jk}=\mu_{j,k}(t)-\mu_{j,k}(\varsigma_j)$;
		\State Let $\underline\mu_{1\to 3} = \min\{\hat\mu_{j1},\hat\mu_{j2},\hat\mu_{j3}\}$ and $\underline \mu_{3\to 5}=\min\{\hat\mu_{j3},\hat\mu_{j4},\hat\mu_{j5}\}$;
		\If{$\underline\mu_{1\to 3}\geq c_2$ and $\min\{\frac{\hat r_{j3}}{\hat\mu_{j3}}-\frac{\hat r_{j2}}{\hat\mu
		_{j2}},\frac{\hat r_{j2}}{\hat\mu_{j2}}-\frac{\hat r_{j1}}{\hat\mu
		_{j1}}\}> \frac{3c_1}{\sqrt{\underline\mu_{1\to 3}}}+\frac{3c_1'}{\underline\mu_{1\to 3}}$} \label{line:cppt-update1}
			\State $\vct\rho_j \gets (\rho_{j2},\frac{1}{4}\rho_{j2}+\frac{3}{4}\rho_{j5},\frac{1}{2}\rho_{j2}+\frac{1}{2}\rho_{j5}, \frac{3}{4}\rho_{j2}+\frac{1}{4}\rho_{j5},\rho_{j5})$, 
			$\varsigma_j\gets t$;
		\ElsIf{$\underline\mu_{3\to 5}\geq c_2$ and $\min\{\frac{\hat r_{j3}}{\hat\mu_{j3}}-\frac{\hat r_{j4}}{\hat\mu
		_{j4}},\frac{\hat r_{j4}}{\hat\mu_{j4}}-\frac{\hat r_{j5}}{\hat\mu
		_{j5}}\}> \frac{3c_1}{\sqrt{\underline\mu_{3\to 5}}}+\frac{3c_1'}{\underline\mu_{3\to 5}}$ \label{line:cppt-update2}} 
			\State $\vct\rho_j \gets (\rho_{j1},\frac{1}{4}\rho_{j1}+\frac{3}{4}\rho_{j4},\frac{1}{2}\rho_{j1}+\frac{1}{2}\rho_{j4}, \frac{3}{4}\rho_{j1}+\frac{1}{4}\rho_{j4},\rho_{j4})$, 
			$\varsigma_j\gets t$;
		\EndIf
	\EndFor
\EndFor
\end{algorithmic}
\end{algorithm}

We first explain the intuitions and design principles behind Algorithm \ref{alg:cppt} without data privacy considerations.
Algorithm \ref{alg:cppt} utilizes two main ideas to carry out nonparametric personalized pricing with demand learning.
The first idea is to partition the space of contextual vectors $\{\vct x_t\}_{t=1}^T\subseteq\mathcal X=[0,1]^d$ into $J$ small hypercubes (denoted by $B_j$ for $j\in[J]$) with equal volume.
The algorithm then treats customers whose context vectors belonging to the same hypercube $B_j$ the same.
The effectiveness of this ``localized'' strategy is justified by the Lipschitz continuity of the expected reward function $f$ (Assumption (A2)) and similar conditions
in Assumption (A3), implying that customers with similar context vectors have similar demand.
Clearly, with more hypercubes (i.e., larger $J$) we are able to approximate customers' demands more accurately as hypercubes become smaller.
On the other hand, we would suffer from the insufficiency of sample size for each hypercube
as we are learning the localized demands of more hypercubes.
An appropriate selection of the number of hypercubes ($J$) is vital for the good performance of the CPPQ policy, which we give more details in Theorem \ref{thm:cppt-regret}, where $J$ is set to be $\lceil {T}^{d/(d+4)}\rceil$.

The second idea in Algorithm \ref{alg:cppt} is to use \emph{quadrisection search} to find the optimal price $p^*(B_j)$ for all customers belonging to the same hypercube $B_j$.
The quadrisection search procedure is justified by Assumption (A3-a), which asserts that the expected reward of all customers belonging to hypercube $B_j$ 
is \emph{strongly concave} with respect to the offered price $p$. While similar bisection or trisection methods have been adopted in the dynamic pricing literature for concave objectives \citep{wang2014close,lei2014near}, the previous work does not require dividing price intervals into four equally sized pieces.
The reason we need to use a quadrisection search approach is because of the following:
when dividing the price interval into three equally sized pieces, the (noisy) comparison of objective values between the two mid-points in this trisection search
approach may be non-conclusive {(see e.g., Figure \ref{fig:trisection}).}
Under normal circumstances, the price interval could still be updated and shrunk in case of non-conclusive comparison 
once a pre-allocated sample budget is consumed.
However, when the pricing platform is subject to privacy constraints (especially the local privacy constraints)
it is difficult for the algorithm to maintain accurate sample counts for each hypercube.
On the other hand, by increasing the number of mid-points to explore, we can ensure that 
at least one sets of conditions in Lines \ref{line:cppt-update1}, \ref{line:cppt-update2} of Algorithm \ref{alg:cppt} are automatically satisfied
when a hypercube receives enough samples.

{In the left panel of Figure \ref{fig:quadrisection}, without noise, if we have $f_j(\rho_{j1})\leq f_j(\rho_{j2})\leq f_j(\rho_{j3})$, this must imply that $p^*(B_j)\geq \rho_{j2}$ by concavity of $f_j(\cdot)$ (i.e., Assumption (A3-a)). Therefore, by shrinking the price range from $[\rho_{j1},\rho_{j5}]$ to $[\rho_{j2},\rho_{j5}]$, we still have $p^*(B_j)$ contained in the new price range. Similarly, the right panel of Figure \ref{fig:quadrisection} shows the case that $f_j(\rho_{j5})\leq f_j(\rho_{j4})\leq f_j(\rho_{j3})$ implies $p^*(B_j)\leq \rho_{j4}$; thus we can shrink $[\rho_{j1},\rho_{j5}]$ to $[\rho_{j1},\rho_{j4}]$ without losing $p^*(B_j)$. When there is noise (from both demand realization and privacy), we still have the same conclusion (with high probability) given the data samples of of all prices are large enough. 
} 

\begin{figure}[t]
	\centering
	\includegraphics[width=0.45\textwidth]{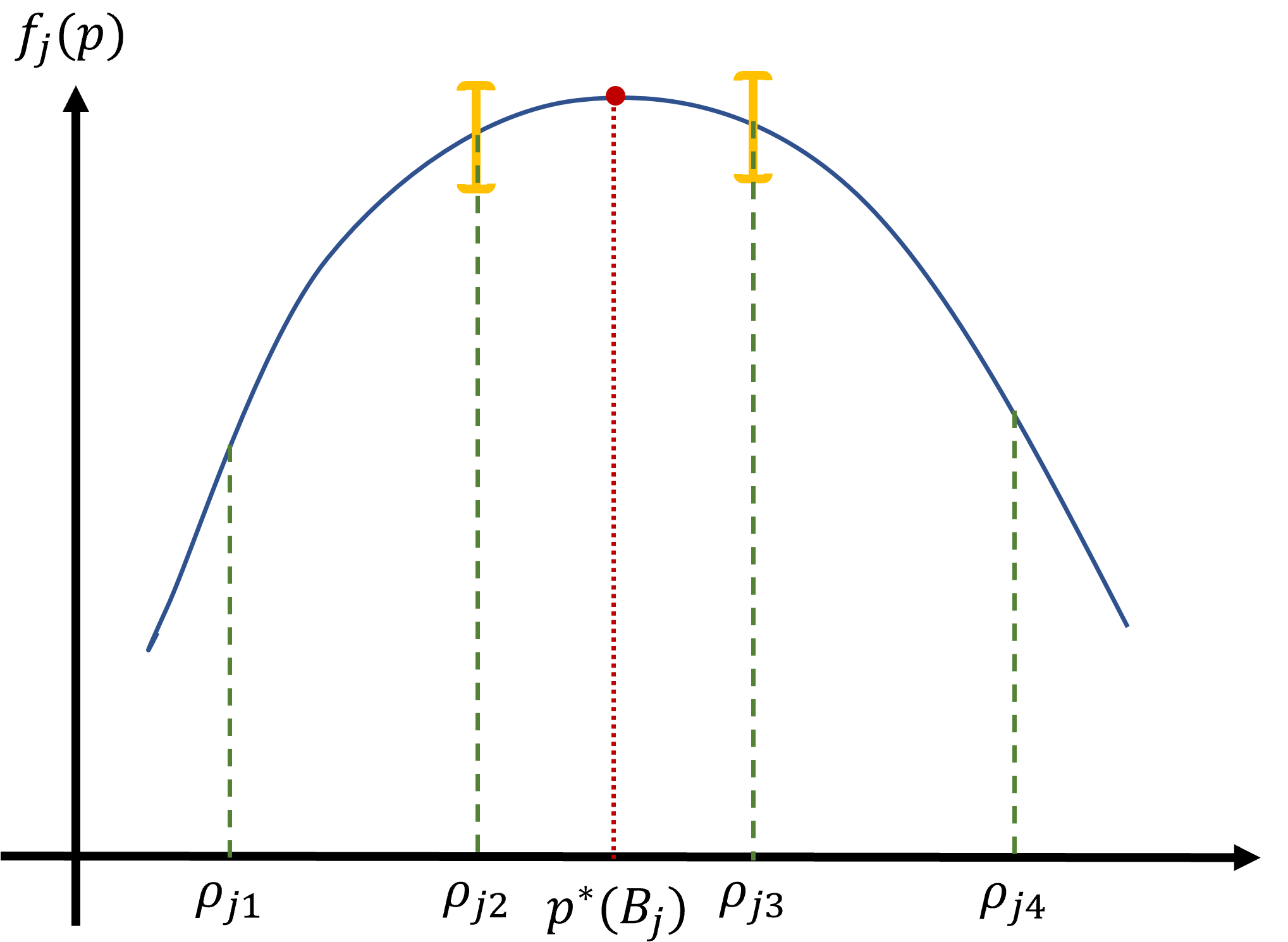}
	\caption{An illustration of trisection search. In this case, we observe similar revenue of two mid-points; hence it is difficult to decide whether it is better to shrink from left or right.}
	\label{fig:trisection}
\end{figure}

\begin{figure}[t]
	\centering
	\includegraphics[width=0.9\textwidth]{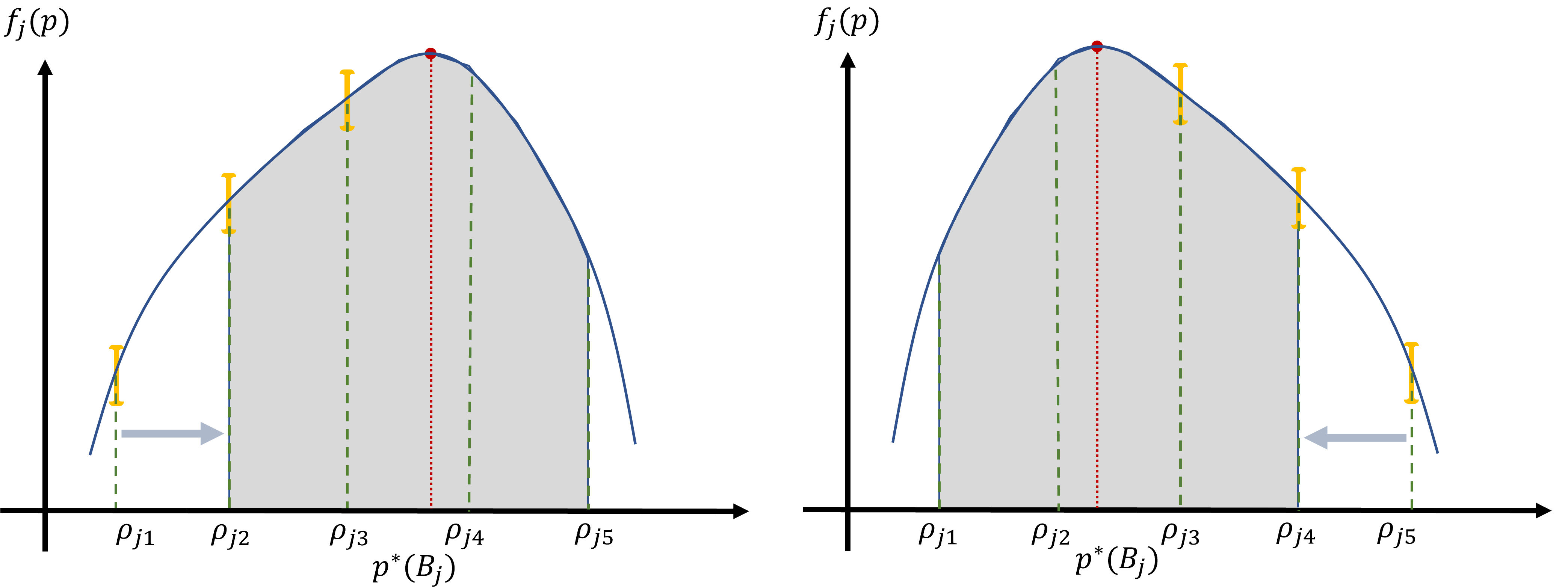}
	\caption{An illustration of quadrisection search. The left/right panel satisfies condition in Line \ref{line:cppt-update1}/Line \ref{line:cppt-update2} of Algorithm \ref{alg:cppt}.}
	\label{fig:quadrisection}
\end{figure}

Note that in the work of \cite{chen2021nonparametric} an alternative multi-armed bandit formulation was adopted, which is asymptotically optimal under non-private settings.
It is however unclear how to  extend such approaches to private settings. As the number of arms in the multi-armed bandit formulation must scale polynomially with time horizon $T$, privatizing all arms will lead to a sub-optimal regret.

It is worth explaining in more details of the construction of the $r_{j,k}(t)$ and $\mu_{j,k}(t)$ values in Algorithm \ref{alg:cppt},
and how they can be used to guide the parallel quadrisection updates of $\vct\rho_j$ in different hypercubes $B_j$.
To gain more intuitions, imagine for now that Algorithm \ref{alg:cppt} is no longer subject to any privacy constraints.
In this ideal scenario, we have $\varepsilon\to \infty$ and as a result $r_{j,k}(t)=\sum_{\tau\leq t}u_{\tau,j,k}$
and $\mu_{j,k}(t) = \sum_{\tau\leq t}v_{\tau,j,k_t}$.
Hence, $r_{j,k}(t)$ corresponds to the cumulative reward received for customers in hypercube $B_j$ to whom price $\rho_{jk}$ is offered,
and $\mu_{j,k}(t)$ is the total number of customers in hypercube $B_j$ to whom price $\rho_{jk}$ is offered.
Since $\vct\rho_j=(\rho_{jk})_{k=1}^r$ is constantly updated in Algorithm \ref{alg:cppt},
the $\varsigma_j$ serves as a ``pointer'' for hypercube $B_j$, meaning that the time periods after $\varsigma_j$ correspond to rewards and customer counts
for the current version of the $\vct\rho_j$ parameter.
With $\hat r_{jk}=r_{jk}(t)-r_{jk}(\varsigma_j)$ and $\hat \mu_{jk}=\mu_{jk}(t)-\mu_{jk}(\varsigma_j)$,
the ratio $\hat r_{jk}/\hat\mu_{jk}$ would then approximate $f_{B_j}(\rho_{jk})$ by the law of large numbers.

Finally, we explain how Algorithm \ref{alg:cppt} protects customers' data privacy in a centralized manner.
The main idea is to calibrate artificial Laplace noise into the cumulative reward $r_{j,k}(t)$ and customer counts $\mu_{j,k}(t)$,
so that the price and realized demand of a single customer will not affect much of the final reward/customer counts.
It is worth noting that the hypercube \emph{index} $j_t$ also reveals sensitive information of $\vct x_t$. {To prevent an adversary from identifying the hypercube $j_t$ that $\vct x_t$ belongs to}, we need to calibrate artificial noise into \emph{all} hypercubes $j=1,\ldots, J$ at each time period $t$, regardless of whether $\vct x_t$ belongs to a particular hypercube or not.   
In addition, to alleviate composition of central differential privacy,
our proposed CPPQ algorithm utilizes a \emph{tree-based aggregation} framework  \citep{chan2011private,dwork2014analyze}. For completeness purposes we provide the pseudo-code description of this aggregation framework 
in Algorithm \ref{alg:tree-based-aggregation}.
The main objective of this procedure is to release sequential private data with provable central privacy guarantees.
More specifically, the $r_{j,k_t}(t)$ statistics constructed in Algorithm \ref{alg:tree-based-aggregation} is a (centralized) privatized version
of the cumulative statistic $\sum_{\tau\leq t}u_{\tau,j,k_t}$, and similarly $\mu_{j,k_t}(t)$ is a privatized version
of the cumulative statistic $\sum_{\tau\leq t}v_{\tau,j,k_t}$.
{We refer the readers to the works of \cite{chan2011private,dwork2014analyze} as well as the recent work of \cite{chen2020privacy} for details
and motivations of this procedure and its analysis.}

\begin{algorithm}[t]
\label{alg:tree-based-aggregation}
\caption{The tree-based aggregation procedure \citep{chan2011private,dwork2014analyze}}
\begin{algorithmic}[1]
\Function{TreeBasedAggregation}{$\{u_\tau\}_{\tau<t}$, $t$, $s_t$, $\varepsilon$, $T$}
	\State ({Initialization}: $\alpha_\ell=\hat\alpha_\ell=0$ for $\ell=0,1,\cdots,\lfloor\log_2 T\rfloor$ when initialized;)
	\State Let $\{\alpha_\ell,\hat\alpha_\ell\}_{\ell=0}^L$ be associated with $\{u_\tau\}_{\tau<t}$, where $L=\lfloor\log_2 T\rfloor$ and $\varepsilon'=\varepsilon/(L+1)$;
	\State Let $t=\sum_{\ell=0}^L b_\ell(t) 2^\ell$ with $b_\ell(t)\in\{0,1\}$ be the binary expression of $t$;
	\State $\ell_{\min}(t):=\min\{\ell: b_\ell(t)=1\}$; 
	\State Update $\alpha_{\ell_{\min}(t)}\gets\sum_{\ell<\ell_{\min}(t)}\alpha_\ell + u_t$
	and $\alpha_\ell=\hat\alpha_\ell=0$ for all $\ell<\ell_{\min}(t)$;
	\State Calibrate noise $\hat\alpha_{\ell_{\min}}(t)\gets\hat\alpha_{\ell_{\min}}(t) + w_t$, where $w_t\sim \mathrm{Lap}(2/\varepsilon')$;
	\State \textbf{return} $\hat U(t)=\sum_{\ell=0}^L b_\ell(t)\hat\alpha_\ell(t)$;
\EndFunction
\end{algorithmic}
\end{algorithm}

Our next theorem proves that the proposed CPPQ policy satisfies $\varepsilon$-CDP, and also establishes an upper bound on the regret incurred by the algorithm.

\begin{theorem}
The CPPQ policy described in Algorithm \ref{alg:cppt} satisfies $\varepsilon$-CDP as defined in Definition \ref{defn:anticipating-dp}.
Furthermore, if Algorithm \ref{alg:cppt} is executed with $J=\lceil {T}^{d/(d+4)}\rceil$, $c_1=\sqrt{\ln(2T^3)}$, $c_1'=4c_2$ and $c_2=76\varepsilon^{-1}\ln^2(2T^3)$,
then the regret of Algorithm \ref{alg:cppt} can be upper bounded by
$$
 \mathbb E\left[\sum_{t=1}^T f(p^*(\vct x_t),\vct x_t)-f(p_t,\vct x_t)\right]\leq \bar C_1\times T^{(d+2)/(d+4)} + \bar C_1'\times \varepsilon^{-1}T^{d/(d+4)} + O(1),
$$
where $\bar C_1,\bar C_1'$ are constants satisfying $\bar C_1\leq \const\times C_H^2(\sigma_H^{-4}+C_H\sqrt{C_X})\ln^2(2C_H^2T^3)+C_H^2C_p^2d/2$ and $\bar C_1' \leq \const\times C_H^2\sigma_H^{-2}\ln^3(2C_H^2T^3)$, where $\const$ are numerical constants that do not depend on any problem parameters.
\label{thm:cppt-regret}
\end{theorem}

\begin{remark}
The $\varepsilon$-CDP privacy guarantee of Algorithm \ref{alg:cppt} holds for \emph{any} values of algorithm parameters $J,c_1,c_1'$ and $c_2$.
This gives practitioners more flexibility in tuning numerical constants in these algorithm parameters for better empirical performances.
\end{remark}

{Since the primary focus of the paper lies on the more practical local privacy  setting,}
we relegate the proof of Theorem \ref{thm:cppt-regret} to the supplementary material.
It is however interesting to discuss the regret upper bound obtained in Theorem \ref{thm:cppt-regret} and contrast it with existing results of \cite{chen2021nonparametric}
under non-privacy settings.
The dominating term (as $T\to\infty$) in Theorem \ref{thm:cppt-regret} is $\widetilde O(T^{(d+2)/(d+4)})$ with the coefficient $\bar C_1$ being independent from
the privacy parameter $\varepsilon$. This term matches the upper and lower bound in \cite{chen2021nonparametric}.
The cost of performance arising from protecting customers' data privacy is reflected in the $\widetilde O(\varepsilon^{-1}T^{d/(d+4)})$ term,
with smaller $\varepsilon$ corresponding to stronger privacy guarantees and therefore larger regret.
However, note that for this term the $T^{d/(d+4)}$ regret will be asymptotically dominated by the $T^{(d+2)/(d+4)}$ regret in the other term,
showing that the impact of privacy constraints will diminish as more customers and their data are available to the platform.
This is a unique feature of the central privacy regime for which the platform has centralized control over the release of sensitive information,
which is also observed in the work of \cite{chen2020privacy} for parametric demand models.
The situation will be completely different for locally private settings, as we shall see in the next section and Theorem \ref{thm:regret-lppt}.

\section{The \textsc{Locally-Private-Parallel-Quadrisection} (LPPQ) Algorithm}\label{sec:ldp}

In this section we describe a personalized pricing algorithm that satisfies $\varepsilon$-LDP.
The proposed algorithm is named \textsc{Locally-Private-Parallel-Quadrisection} (LPPQ) and its pseudocode description is given in Algorithm \ref{alg:lppt}.

\begin{algorithm}[t]
\caption{The \textsc{Locally-Private-Parallel-Quadrisection} (LPPQ) algorithm}
\label{alg:lppt}
\begin{algorithmic}[1]
\State \textbf{Input}: time horizon $T$, privacy parameter $\varepsilon$, algorithm parameters $J,\kappa_1,\kappa_2$.
\State Initialize: partition $[0,1]^d$ into $J$ equally-sized hypercubes (each side's length being $h=J^{-1/d}$);
for each hypercube $B_j$, $j\in[J]$, let $\vct\rho_j=(\underline p,\frac{1}{4}\underline p+\frac{3}{4}\overline p, \frac{1}{2}\underline p+\frac{1}{2}\overline p,\frac{3}{4}\underline p+\frac{1}{4}\overline p,\overline p)\in[\underline p,\overline p]^5$, $\vct r_j(0)=(0,0,0,0,0)$, $\varsigma_j=0$;
\For{$t=1,2,\cdots,T$}
	\State Receive $\vct x_t\in[0,1]^d$ and let $j_t\in[J]$ be an integer such that $\vct x_t\in B_{j_t}$;
	\State Offer price $p_t=\rho_{j_tk_t}$ where $k_t\equiv t\mod 5$, and receive $y_t\in[0,1]$;
	\For{$j=1,2,\cdots,J$}
		\State $r_{j,k_t}(t)\gets r_{j,k_t}(t-1) + \vct 1\{j=j_t\}p_ty_t + w_{j,t}$, $w_{j,t} \overset{i.i.d }{\sim}\mathrm{Lap}(2/\varepsilon)$; 
		\State $r_{j,k}(t)\gets r_{j,k}(t-1)$ for $k\neq k_t$;
		\State For $k\in[5]$ compute $\hat r_{jk}=r_{j,k}(t)-r_{j,k}(\varsigma_j)$; let $n_j\gets t- \varsigma_j$;
		\If{$n_j\geq \kappa_2$ and $\min\{\hat r_{j2}-\hat r_{j1},\hat r_{j3}-\hat r_{j2}\}/(5h^dn_j) > 3\kappa_1/(\varepsilon h^d\sqrt{n_j})$\label{line:lppt-update1}}
			\State $\vct\rho_j \gets (\rho_{j2},\frac{1}{4}\rho_{j2}+\frac{3}{4}\rho_{j5},\frac{1}{2}\rho_{j2}+\frac{1}{2}\rho_{j5}, \frac{3}{4}\rho_{j2}+\frac{1}{4}\rho_{j5},\rho_{j5})$, 
			$\varsigma_j\gets t$;
		\ElsIf{$n_j\geq \kappa_2$ and $\min\{\hat r_{j3}-\hat r_{j4},\hat r_{j4}-\hat r_{j5}\}/(5h^dn_j) > 3\kappa_1/(\varepsilon h^d\sqrt{n_j})$ \label{line:lppt-update2}}
			\State  $\vct\rho_j \gets (\rho_{j1},\frac{1}{4}\rho_{j1}+\frac{3}{4}\rho_{j4},\frac{1}{2}\rho_{j1}+\frac{1}{2}\rho_{j4}, \frac{3}{4}\rho_{j1}+\frac{1}{4}\rho_{j4},\rho_{j4})$, 
			$\varsigma_j\gets t$;;
		\EndIf
	\EndFor
\EndFor
\end{algorithmic}
\end{algorithm}

The main ideas and principles of the LPPQ algorithm are the same as the CPPQ algorithm presented in the previous section:
the algorithm first partitions the contextual vector space $\mathcal X=[0,1]^d$ into $J$ equally sized small hypercubes\footnote{The choice of $J$ is, however, different from the CPPQ algorithm. See Theorems \ref{thm:regret-lppt} and \ref{thm:cppt-regret} for more details.}
and then uses quadrisection search to localize the optimal price $p^*(B_j)$ (see \eqref{eq:opt_cube}) for all customers whose contextual vectors belong to hypercube $B_j$.
The major difference between CPPQ and LPPQ lies in how the privatized statistics $\hat r_{jk},\hat\mu_{jk}$ are constructed,
as central and local privacy impose different constraints on the platform's side.
Below we explain the difference in details:
\begin{enumerate}
\item In the central privacy setting, the per-period instantaneous statistics $u_{j,t,k}=\vct 1\{\vct x_t\in B_j\wedge k_t=k\}p_ty_t$ and $v_{j,t,k}=\vct 1\{\vct x_t\in B_j\wedge k_t=k\}$ involving sensitive data are in complete possession of the pricing platform.
The platform calibrates noise whenever it needs to release statistics $r_{j,k}(t)$ or $\mu_{j,k}(t)$ for pricing or model update purposes.
In this way, the pricing platform has full knowledge of each customer's sensitive data, but third-party malicious agents would not be able to recover
these sensitive data through interactions with the platform.
It can also be shown (see, e.g., Lemma \ref{lem:cppt-concentration} in the supplementary material) that the magnitude of noise calibrated into $r_{j,k}(t)$ and $\mu_{j,k}(t)$
is on the order of $O(\varepsilon^{-1}\ln^2 T)$, a relatively small level since the cumulative statistics $r_{j,k}(t)$ and $\mu_{j,k}(t)$ could be as large as $n_j$.

On the other hand, in the local privacy setting, the per-period instantaneous statistics $\vct 1\{\vct x_t\in B_j\wedge k_t=k\}p_ty_t$ is directly privatized by a Laplace noise 
and then communicated to the pricing platform. In this way, the platform keeps no copy of any customer's sensitive data and the system/protocol
is therefore more secure compared to CPPQ.
The downside is that the noise calibrated to the privatized statistics is on the order of $\widetilde O(\varepsilon^{-1}\sqrt{n_j})$ (see e.g. Lemma \ref{lem:ucb-validity} in the coming sections), significantly larger than
$O(\varepsilon^{-1}\ln^2 T)$ achievable for central privatizing mechanisms.

\item Because the calibrated noise magnitude is too large in local privacy settings, it is no longer feasible to keep track of the customer count statistics $\hat\mu_{jk}$ like CPPQ does as the signals in the customer counts are too weak. Instead, in the LPPQ algorithm we directly use $n_j=t-\varsigma_j$ {(i.e., the total number of time periods after the previous pointer $\varsigma_j$)} to construct confidence intervals.
This difference is more explicit if one compares the conditions of Lines \ref{line:cppt-update1}, \ref{line:cppt-update2} in Algorithm \ref{alg:cppt}
with those in Algorithm \ref{alg:lppt}, as the confidence intervals in the former conditions are constructed using privatized customer counts $\hat\mu_{jk}$
while the confidence intervals in the conditions of Lines \ref{line:lppt-update1}, \ref{line:lppt-update2} in Algorithm \ref{alg:lppt}
are built directly using the total time period counts $n_j$.
\end{enumerate}

As local privacy is the main focus of this paper, we present detailed analysis of the privacy and regret performance guarantees of the proposed LPPQ policy
in the next two sections. We also present an information theoretical lower bound in Section~\ref{sec:lower-bound} that nicely complements
our regret upper bound in Theorem \ref{thm:regret-lppt}.

\subsection{Privacy Analysis}


To see the LPPQ policy satisfies $\varepsilon$-LDP, we first explain how the policy can be parameterized as $\{Q_t,A_t\}_{t=1}^T$
as defined in Eqs.~(\ref{eq:defn-Qt},\ref{eq:defn-At}).
For each time period $t$, the intermediate variable $\vct z_t\in\mathbb R^J$ is produced as $z_{tj}=\vct 1\{j=j_t\}p_ty_t + w_{j,t}$,
where $w_{j,t}\overset{i.i.d}{\sim}\mathrm{Lap}(2/\varepsilon)$.
Clearly, the distribution of $\vct z_t$ is measurable conditioned on $s_t=(\vct x_t,y_t,p_t)$, satisfying Eq.~(\ref{eq:defn-Qt}).
With $\{\vct z_1,\cdots,\vct z_{t-1}\}$, the algorithm can compute the values of $\{\vct\rho_j,\vct r_j,\varsigma_j,n_j\}_{j=1}^J$ without accessing
any of $s_1,\cdots,s_{t-1}$.
The offered price $p_t$ only depends on $k_t$, $\vct x_t$ and $\{\vct\rho_j\}_{j=1}^J$.
Hence, $A_t$ is measurable conditioned on $\vct x_t$ and $\vct z_{<t}$, satisfying Eq.~(\ref{eq:defn-At}).
The following proposition then follows immediately:
\begin{proposition}
The LPPQ policy described in Algorithm \ref{alg:lppt} satisfies $\varepsilon$-LDP.
\label{prop:lppt-privacy}
\end{proposition}
\begin{proof}{Proof of Proposition \ref{prop:lppt-privacy}.}
Fix arbitrary $t$. Because $y_t\in\mathcal Y\subseteq[0,1]$ and $p_t\in[\underline p,\overline p]\subseteq[0,1]$,
the $\ell_1$-sensitivity of $y_tp_t\vct e_{j_t}\in\mathbb R^J$ is upper bounded by 2 (that is, $\sup_{s_t,s_t'}\|y_tp_t\vct e_{j_t}-y_t'p_t'\vct e_{j_t'}\|_1\leq 2$).
Hence, the $\varepsilon$-LDP of $z_{tj}=\vct 1\{j=j_t\}p_ty_t + w_{j,t}$, $j\in[J]$, $w_{j,t}\overset{i.i.d.}{\sim}\Lap(2/\varepsilon)$
is guaranteed by the Laplace mechanism \citep{dwork2014algorithmic}. $\square$
\end{proof}

\subsection{Regret Analysis}

The main objective of this section is to establish the following theorem upper bounding the cumulative regret of LPPQ
when algorithm parameters are carefully chosen.
\begin{theorem}
Suppose Assumptions (A1) to (A3) hold, and Algorithm \ref{alg:lppt} is executed with $J=\lceil (\varepsilon\sqrt{T})^{d/(d+2)}\rceil$,
$\kappa_1=1.7\sqrt{\ln(2T)}$ and $\kappa_2=31\ln T$. Then the regret of Algorithm \ref{alg:lppt} can be upper bounded by
\begin{align*}
\mathbb E\left[\sum_{t=1}^T f(p^*(\vct x_t),\vct x_t)-f(p_t,\vct x_t)\right]
&\leq \overline C_2\times \varepsilon^{-2/(d+2)} T^{(d+1)/(d+2)},
\end{align*}
where $\overline C_2:= \sigma_H^{-1}{(160\kappa_1+C_X\sqrt{\kappa_2})\sqrt{2\ln(2C_XC_L T)}}+0.5C_H^2C_p^2C_X d$.
\label{thm:regret-lppt}
\end{theorem}

\begin{remark}
The $\varepsilon$-LDP privacy guarantee of Algorithm \ref{alg:lppt} holds for \emph{any} values of algorithm parameters $J,\kappa_1$ and $\kappa_2$.
This gives practitioners more flexibility in tuning numerical constants in these algorithm parameters for better empirical performances.
\end{remark}

\begin{remark}
In the setting where customers' personal data privacy is not of concerns, the work of \cite{chen2021nonparametric} achieves
a cumulative regret upper bound of $\widetilde O(T^{(d+2)/(d+4)})$. In contrast, the regret upper bound in Theorem \ref{thm:regret-lppt}
is on the order of $O(\varepsilon^{-2/(d+2)}T^{(d+1)/(d+2)})$, which is a polynomial factor worse even if the privacy parameter $1/\varepsilon$ is on a constant level.
Such a performance guarantee, while seemingly undesirable, is the best one can achieve, as we show in Theorem \ref{thm:regret-lower-bound} in the next section
that no locally differentially private policy can achieve regret significantly smaller than $O(T^{(d+2)/(d+4)})$.
\end{remark}

In the rest of this section we prove Theorem \ref{thm:regret-lppt}.
We first establish a technical lemma showing that the $\hat r_{jk}$ values are faithful estimates of $f_{B_j}(\cdot)$ evaluated on 
price vectors $\vct\rho_j$.
\begin{lemma}
For $j\in[J]$ define $\bar\chi(B_j) := J\times \Pr_{\vct x\sim P_X}[\vct x\in B_j]$.
With probability $1-O(T^{-1})$ the following holds uniformly for all $t$, $j\in[J]$ and $k\in[5]$ that satisfies $n_j\geq \kappa_2$:
$$
\left|\frac{\hat r_{jk}}{5h^d n_j} - \bar\chi(B_j)f_{B_j}(\rho_{jk})\right| \leq \frac{\kappa_1}{\varepsilon h^d\sqrt{n_j}}.
$$
\label{lem:ucb-validity}
\end{lemma}
\begin{proof}{Proof of Lemma \ref{lem:ucb-validity}.}
Fix $j,k$ and a particular time period $t$.
Note that $J=h^{-d}$.
Without loss of generality assume the time periods are $t=k,5+k,\cdots,t_j$ where $t_j$ is the largest integer not exceeding $n_j$
that is equiavlent to $k$ modulo 5, and $\hat r_{jk}=\sum_t u_{j,t}$,
where $u_{j,t}=\vct 1\{j_t=j\}y_tp_t + w_{j,t}$, $w_{j,t}\overset{i.i.d.}{\sim}\Lap(2/\varepsilon)$.
It then holds that
$
u_{j,t} = \mu_j + \xi_{j,t} + w_{j,t},
$
where $\mu_j = \bar\chi(B_j)h^df_{B_j}(\rho_{jk})$, $\mathbb E[\xi_{j,t}]=\mathbb E[w_{j,t}]=0$ and $|\xi_{j,t}|\leq 1$ almost surely.
Furthermore, $\{\xi_{j,t},w_{j,t}\}_t$ are independent random variables.
By Hoeffding's inequality \citep{hoeffding1963probability} and the fact that $\sum_t\xi_{j,t}$ is the sum of $n_j/5$ independently distributed centered random variables, 
we have with probability $1-O(T^{-3})$ that
\begin{equation}
\left|\sum_t\xi_{j,t}\right| \leq 1.2\sqrt{n_j\ln(2T)}.
\label{eq:proof-ucb-validity-1}
\end{equation}
For the $\sum_t w_{j,t}$ term, invoke \citep[Lemma 2.8]{chan2011private} for concentration inequalities of independently distributed Laplace random variables.
We have that for all $n_j\geq 31\ln T$, with probability $1-O(T^{-3})$ it holds that
\begin{equation}
\left|\sum_t w_{j,t}\right|\leq 7\varepsilon^{-1}\sqrt{n_j\ln(2T)}.
\label{eq:proof-ucb-validity-2}
\end{equation}
Combining Eqs.~(\ref{eq:proof-ucb-validity-1},\ref{eq:proof-ucb-validity-2}), we have with probability $1-O(T^{-3})$ that
\begin{align*}
\bigg|\frac{\hat r_{jk}}{5h^d n_j} - \bar\chi(B_j)f_{B_j}(\rho_{jk})\bigg| 
&\leq \frac{1}{5h^dn_j}\left[1.2\sqrt{n_j\ln(2T)} + \frac{7\sqrt{n_j\ln(2T)}}{\varepsilon}\right]
\leq \frac{1.7\sqrt{\ln(2T)}}{\varepsilon h^d\sqrt{n_j}}.
\end{align*}
Lemma \ref{lem:ucb-validity} is then proved by applying the union bound over all $t$ and $j\in[J]$. $\square$
\end{proof}

With Lemma \ref{lem:ucb-validity}, the concavity of $f_{B_j}(\cdot)$ immediately yields the following corollary:
\begin{corollary}
With probability $1-O(T^{-1})$ it holds for all $j\in[J]$ and $t$ that $p^*(B_j)\in[\rho_{j1},\rho_{j5}]$.
\label{cor:ucb-validity}
\end{corollary}

To introduce the next key technical lemma we need to define some notations.
For a hypercube $B_j$, $j\in[J]$, we partition the entire $T$ selling periods into \emph{epochs} denoted as $\tau=1,2,3,\cdots$,
with each epoch starting with a time period at which $\varsigma_j$ is reset (at the start of $T$ time periods or as a result of the execution of Line \ref{line:lppt-update1}
or \ref{line:lppt-update2} in Algorithm \ref{alg:lppt}),
and ending when either Line \ref{line:lppt-update1} or Line \ref{line:lppt-update2} is executed again to reset the $\varsigma_j$ pointer.
Let $\mathcal T_j(\tau)$ be the collection of time periods during epoch $\tau$ for hypercube $j$, and define $n_j(\tau) := |\mathcal T_j(\tau)|$.
Note that $\mathcal T_j(\tau)$ includes selling periods during which customers with $\vct x_t\notin B_j$ arrive as well.
The following technical lemma upper bounds the cumulative regret incurred by customers with $\vct x_t\in B_j$ during epoch $\tau$.
\begin{lemma}
Fix hypercube $B_j$, $j\in[J]$ and let $\tau$ be an epoch. Then conditioned on the success event in Lemma \ref{lem:ucb-validity}, it holds that
\begin{equation}
n_j(\tau) \geq \frac{\kappa_1^2}{\varepsilon^2h^{2d}C_X^2C_L^2\delta_\tau^2},
\label{eq:nj-epoch}
\end{equation}
where $\delta_\tau=(\overline p-\underline p)\cdot (3/4)^{\tau-1}$. Furthermore, 
\begin{multline}
\mathbb E\left[\sum_{t\in\mathcal T_j(\tau)} \vct 1\{\vct x_t\in B_j\}\big(f(p^*(\vct x_t),\vct x_t)-f(p_t,\vct x_t)\big) \right]\\
\leq 
\frac{64\kappa_1+C_X\sqrt{\kappa_2}}{\sigma_H\varepsilon}\mathbb E\left[\sqrt{n_j(\tau)}\right] + \frac{1}{2}C_H^2C_p^2C_X dh^{d+2}\mathbb E[n_j(\tau)].
\label{eq:regret-epoch}
\end{multline}
\label{lem:regret-epoch}
\end{lemma}
\begin{proof}{Proof of Lemma \ref{lem:regret-epoch}.}
Decompose the difference $f(p^*(\vct x_t),\vct x_t)-f(p_t,\vct x_t)$ as
$$
f(p^*(\vct x_t),\vct x_t)-f(p_t,\vct x_t)
= \big[f(p^*(\vct x_t),\vct x_t) - f(p^*(B_j),\vct x_t)\big] + \big[f(p^*(B_j),\vct x_t) - f(p_t,\vct x_t)\big].
$$
To upper bound the first term, invoke Assumption (A3-a) with $B=\{\vct x_t\}$, and Assumptions (A3-b), (A3-c) with $B=B_j$. We have
\begin{align*}
f(p^*(\vct x_t),\vct x_t)-f(p^*(B_j),\vct x_t) \leq \frac{C_H^2}{2}\big|p^*(\vct x_t)-p^*(B_j)\big|^2 
\leq \frac{C_H^2C_p^2}{2}\sup_{\vct x,\vct x'\in B_j}\|\vct x-\vct x'\|_2^2\leq \frac{C_H^2C_p^2d}{2} h^2.
\end{align*}
Subsequently, the left-hand side of Eq.~(\ref{eq:regret-epoch}) can be upper bounded by
\begin{equation}
\mathbb E\left[\sum_{t\in\mathcal T_j(\tau)} \vct 1\{\vct x_t\in B_j\}\big(f_{B_j}(p^*(B_j))-f_{B_j}(p_t)\big) \right] + \frac{C_H^2C_p^2d}{2} h^2 n_j(\tau)\Pr[\vct x\in B_j].
\label{eq:proof-regret-epoch-1}
\end{equation}

In epoch $\tau$, define $\delta_\tau := \rho_{j5}-\rho_{j1}=(\overline p-\underline p)\cdot (3/4)^{\tau-1}$ where the equality is by update rule of price range. Since $p^*(B_j)\in[\rho_{j1},\rho_{j5}]$ thanks to Corollary \ref{cor:ucb-validity}, and $f_{B_j}(\cdot)$ is $\sigma_H$-strongly concave
thanks to Assumption (A3-a) with $B=B_j$, we have that either $\min\{f_{B_j}(\rho_{j2})-f_{B_j}(\rho_{j1}),f_{B_j}(\rho_{j3})-f_{B_j}(\rho_{j2})\}\geq \frac{\sigma_H^2}{32}\delta_\tau^2$
(if $p^*(B_j)\geq\rho_{j3}$), or $\min\{f_{B_j}(\rho_{j3})-f_{B_j}(\rho_{j4}),f_{B_j}(\rho_{j4})-f_{B_j}(\rho_{j5})\}\geq \frac{\sigma_H^2}{32}\delta_\tau^2$ (if $p^*(B_j)\leq \rho_{j3}$).
Without loss of generality assume $p^*(B_j)\geq\rho_{j3}$ and $\min\{f_{B_j}(\rho_{j2})-f_{B_j}(\rho_{j1}),f_{B_j}(\rho_{j3})-f_{B_j}(\rho_{j2})\}\geq \frac{\sigma_H^2}{32}\delta_\tau^2$.
By Lemma \ref{lem:ucb-validity}, this implies that throughout the epoch $\tau$, 
$$
5\times \frac{\kappa_1}{\varepsilon h^d\sqrt{n_j}} \geq \bar\chi(B_j)(f_{B_j}\min\{f_{B_j}(\rho_{j2})-f_{B_j}(\rho_{j1}),f_{B_j}(\rho_{j3})-f_{B_j}(\rho_{j2})\} \geq \frac{\sigma_H^2}{32}\bar\chi(B_j)\delta_\tau^2,
$$
where $\bar\chi(B_j)=J\times \Pr[\vct x\in B_j]\in[0,C_X]$, thanks to Assumption (A1), and the first inequality is by Lemma \ref{lem:ucb-validity} and the fact that in any time period of $\tau$ before its end, Line \ref{line:lppt-update1} is not executed. Inverting the above inequality and noting that $n_j(\tau)\geq\kappa_2$ almost surely, we have
\begin{equation}
n_j(\tau) \leq \max\left\{\kappa_2,\frac{25600\kappa_1^2}{\sigma_H^2\varepsilon^2 h^{2d}\bar\chi(B_j)^2\delta_\tau^4}\right\}.
\label{eq:proof-regret-epoch-2}
\end{equation} 

Again within epoch $\tau$ and recall the definition that $\delta_\tau = \rho_{j5}-\rho_{j1}$.
Because $f(\cdot)$ is $C_L$-Lipschitz continuous thanks to Assumption (A2), we have that
\begin{align}
\max_k\big|f_{B_j}(\rho_{j,k+1})-f_{B_j}(\rho_{jk})\big| &\leq C_L\max_k\big|\rho_{j,k+1}-\rho_{jk}| \leq C_L\delta_\tau.\label{eq:proof-regret-epoch-e1}
\end{align}
This implies that $n_j(\tau)$ must satisfy
$$
\frac{\kappa_1}{\varepsilon h^d\sqrt{n_j(\tau)}}\leq \bar\chi(B_j)C_L\delta_\tau,
$$
which yields $n_j(\tau)\geq \frac{\kappa_1^2}{\varepsilon^2h^{2d}\bar\chi(B_j)^2C_L^2\delta_\tau^2} \geq \frac{\kappa_1^2}{\varepsilon^2h^{2d}C_X^2C_L^2\delta_\tau^2}$. This completes the proof of Eq.~(\ref{eq:nj-epoch}).

Additionally, because $p^*(B_j)\in[\rho_{j1},\rho_{j5}]$, and $f_{B_j}(\cdot)$ is twice continuously differentiable with its second derivative bounded by $C_H^2$
and $f'_{B_j}(p^*(B_j)) = 0$ since $p^*(B_j)$ is an interior maximizer of $f_{B_j}(\cdot)$, we have that
\begin{align}
f_{B_j}(p^*(B_j)) - f_{B_j}(p_t) \leq \frac{C_H^2}{2}\delta_\tau^2\leq \frac{160\kappa_1}{\sigma_H\varepsilon h^d\bar\chi(B_j)}\frac{1}{\sqrt{n_j(\tau)}} + \frac{C_H^2\sqrt{\kappa_2}}{2\sqrt{n_j(\tau)}},
\label{eq:proof-regret-epoch-3}
\end{align}
where the last inequality holds by inverting Eq.~(\ref{eq:proof-regret-epoch-2}).
Subsequently, Eq.~(\ref{eq:proof-regret-epoch-1}) can be upper bounded by the expectations of
\begin{align*}
\Pr[\vct x\in B_j]\times& \left[\frac{160\kappa_1\sqrt{n_j(\tau)}}{\sigma_H\varepsilon h^d\bar\chi(B_j)} + \frac{1}{2}\sqrt{\kappa_2n_j(\tau)}\right]+ \frac{C_H^2C_p^2d}{2} h^2 n_j(\tau)\Pr[\vct x\in B_j]\\
&= h^d\bar\chi(B_j)\times  \left[\frac{160\kappa_1\sqrt{n_j(\tau)}}{\sigma_H\varepsilon h^d\bar\chi(B_j)} + \frac{1}{2}\sqrt{\kappa_2 n_j(\tau)}\right] + \frac{C_H^2C_p^2d}{2} h^2 n_j(\tau)\times h^d\bar\chi(B_j)\\
&\leq \frac{160\kappa_1+C_X\sqrt{\kappa_2}}{\sigma_H\varepsilon}\sqrt{n_j(\tau)} + \frac{1}{2}C_H^2C_p^2C_X dh^{d+2}n_j(\tau).
\end{align*}
This completes the proof of Lemma \ref{lem:regret-epoch}. $\square$
\end{proof}

We are now ready to prove Theorem \ref{thm:regret-lppt}.
\begin{proof}{Proof of Theorem \ref{thm:regret-lppt}.}
The entire proof is conditioned on the success event of Lemma \ref{lem:ucb-validity}, with an extra $O(1)$ term in the upper bound of the regret
since the failure probability is $O(T^{-1})$ and in the event of failure the cumulative regret of Algorithm \ref{alg:lppt} is at most $O(T)$.

For each $j\in[J]$, the result in Lemma \ref{lem:regret-epoch} establishes that
\begin{align}
\mathbb E& \left[\sum_{t=1}^T \vct 1\{\vct x_t\in B_j\}\big(f(p^*(\vct x_t),\vct x_t)-f(p_t,\vct x_t)\big) \right]\nonumber\\
&\leq \frac{160\kappa_1+C_X\sqrt{\kappa_2}}{\sigma_H\varepsilon}\mathbb E\left[\sum_{\tau}\sqrt{n_j(\tau)}\right] + \frac{1}{2}C_H^2C_p^2C_X dh^{d+2}\mathbb E\left[\sum_\tau n_j(\tau)\right]\nonumber\\
&\leq \frac{(160\kappa_1+C_X\sqrt{\kappa_2})\sqrt{2\ln(2C_XC_L T)}}{\sigma_H}\frac{\sqrt{T}}{\varepsilon} + \frac{1}{2}C_H^2C_p^2C_X d\times h^{d+2}T.
\label{eq:proof-regret-lppt-1}
\end{align}
where the last inequality holds by invoking the Cauchy-Schwarz inequality and noting that $\sum_\tau n_j(\tau)\leq T$, and that the total number of epochs for each $B_j$ is upper bounded by $2\ln(2C_XC_LT)$, thanks to Eq.~(\ref{eq:nj-epoch})
in Lemma \ref{lem:regret-epoch}.
Define $C_1' := \sigma_H^{-1}{(160\kappa_1+C_X\sqrt{\kappa_2})\sqrt{2\ln(2C_XC_L T)}}$ and $C_2' := 0.5C_H^2C_p^2C_X d$.
Summing both sides of Eq.~(\ref{eq:proof-regret-lppt-1}) over all hypercubes $j\in[J]$, and noting that $J=h^{-d}$, we have
\begin{align}
\mathbb E&\left[\sum_{t=1}^T f(p^*(\vct x_t),\vct x_t)-f(p_t,\vct x_t)\right] \leq C_1'\frac{\sqrt{T}}{h^d\varepsilon} + C_2'h^2 T.\label{eq:proof-regret-lppt-2}
\end{align}
With $J=\lceil (\varepsilon\sqrt{T})^{d/(d+2)}\rceil$ and $h=J^{-1/d}\approx(\varepsilon\sqrt{T})^{-1/(d+2)}$, Eq.~(\ref{eq:proof-regret-lppt-2})
yields the results in Theorem \ref{thm:regret-lppt}, with $\overline C_2=C_1'+C_2'$. $\square$
\end{proof}

\section{Lower Bound of Algorithms with $\varepsilon$-LDP}\label{sec:lower-bound}

In this section we establish the following lower bound, showing that the $\widetilde O(T^{(d+1)/(d+2)})$ regret obtained in Theorem \ref{thm:regret-lppt}
is minimax optimal when the LDP parameter $\varepsilon$ is finite.
More specifically, we will prove the following result:
\begin{theorem}
Let $\pi=\{Q_t,A_t\}_{t=1}^T$ be any personalized pricing policy that satisfies $\varepsilon$-LDP for some $\varepsilon\in(0,1]$.
Let $P_X$ be the uniform distribution on $[0,1]^d$.
Then for $T=\Omega(\varepsilon^{-1}d^{d/2+1}\ln(d/\varepsilon))$, there exists $\lambda(\cdot,\cdot)$ and its associated revenue function $f(\cdot,\cdot)$ satisfying Assumptions (A1) through (A3) with $C_X=1$,
$C_L=4$, $\sigma_H=\sqrt{2}$, $C_H=2$ and $C_p=1$, such that
$$
\mathbb E^\pi\left[\sum_{t=1}^T f(p^*(\vct x_t),\vct x_t)-f(p_t,\vct x_t)\right]\geq \underline C\times \frac{\varepsilon^{-2/(d+2)}T^{(d+1)/(d+2)}}{d^{7/3}},
$$ 
where $\underline C>0$ is a universal numerical constant.
\label{thm:regret-lower-bound}
\end{theorem}

To prove Theorem \ref{thm:regret-lower-bound}, we use similar construction of adversarial instances as in the work of \cite{chen2021nonparametric},
but with different analytical tools such as the Assouad's method \citep{assouad1983deux,yu1997assouad} and 
strong data processing inequalities as consequences of local privacy constraints, as developed in the work of \cite{duchi2018minimax}.

\subsection{Construction of Adversarial Problem Instances}

We adopt the same construction of adversarial problem instances as in the work of \cite{chen2021nonparametric}.
For readers who are not familiar with the construction, we re-capture it here in this section for completeness purposes.
Suppose $[0,1]^d$ is being partitioned into $J$ equally-sized hypercubes, each of length $h=J^{-1/d}$, with $J$ being specified later in the proof.
Let $\{B_j\}_{j=1}^J$ be the $J$ hypercubes that partition $[0,1]^d$.
For each vector $\vct\nu\in\{0,1\}^J$, define problem instance $P_{\vct\nu}$ associated with demand model $\lambda_{\vct\nu}$ as
\begin{equation}
\lambda_{\vct\nu}(p,\vct x) := \frac{2}{3} - \frac{p}{2} + \sum_{j=1}^J\nu_j\left(\frac{1}{3}-\frac{p}{2}\right)\sd(\vct x,\partial B_j),
\label{eq:defn-lambda-nu}
\end{equation}
where $\sd(\vct x,\partial B_j) := \inf_{\vct y\in\partial B_j}\|\vct x-\vct y\|_2$.
It is proved in \citep[Proposition 2]{chen2021nonparametric} that all $\lambda_{\vct\nu}$ and their associated revenue functions $f_{\vct\nu}(p,\vct x)=p\lambda_{\vct\nu}(p,\vct x)$ satisfy Assumptions (A2) and (A3) with $C_L=4$, $\sigma_H=\sqrt{2}$, $C_H=2$ and $C_p=1$.

The demands $\{y_t\}_{t=1}^T$ are stochastically realized as $\Pr_{\vct\nu}[y_t=1|p_t,\vct x_t]=\lambda_{\vct\nu}(p_t,\vct x_t)$
and $\Pr_{\vct\nu}[y_t=0|p_t,\vct x_t] = 1-\lambda_{\vct\nu}(p_t,\vct x_t)$.
It is easy to verify that $\mathbb E_{\vct\nu_t}[y_t|p_t,\vct x_t]=\lambda_{\vct\nu}(p_t,\vct x_t)$.
With $P_X$ being the uniform distribution on $[0,1]^d$ and $[\underline p,\overline p]=1$, all constructed problem instances
satisfy Assumption (A1) with $C_X=1$, $[\underline p,\overline p]=[0,1]$ and $\mathcal Y=\{0,1\}\subseteq[0,1]$.

While the construction of the adversarial problem instances are the same with the work of \cite{chen2021nonparametric},
the analysis of ``distinguishability'' between problem instances are significantly different from \citet{chen2021nonparametric}.
More specifically, when the personalized pricing policy is subject to $\varepsilon$-LDP constraints,
we need much sharper upper bounds on the distinguishability between problem instances in order to derive the $\Omega(T^{(d+1)/(d+2)})$ minimax lower bound,
which is larger by polynomial factors of $T$ compared to the $\Omega(T^{(d+2)/(d+4)})$ regret lower bound proved in \cite{chen2021nonparametric}.
The sharper lower bound also requires us to use different choices of $J$ compared to the arguments in \citep{chen2021nonparametric}.
More details of the analysis is given in subsequent subsections below.

\subsection{Reduction to Classification and Assouad's Lemma}

Recall the definitions that $s_t=(\vct x_t,y_t,p_t)\in\mathcal S=\mathcal X\times\mathcal Y\times[\underline p,\overline p]$
and $z_1,z_2,\cdots,z_T\in\mathcal Z$ are intermediate quantities satisfying $\varepsilon$-LDP, as defined in Eqs.~(\ref{eq:defn-Qt},\ref{eq:defn-At})
and Definition \ref{defn:ldp}.
Our first technical lemma shows that, for the problem instances $\{P_{\vct\nu}\}_{\vct\nu\in\{0,1\}^J}$ constructed in the above section,
a personalized pricing policy with low worst-case regret over $\{P_{\vct\nu}\}_{\vct\nu\in\{0,1\}^J}$ can also identify the values of $\{\nu_j\}$ with small error.
\begin{lemma}
Suppose $T\geq 6.2\sqrt{\ln(2T)}$ and $\sqrt{T}\geq 66\varepsilon^{-1}J\sqrt{\ln(2T)}$.
Let $\pi=\{Q_t,A_t\}_{t=1}^T$ be a personalized pricing policy that satisfies $\varepsilon$-LDP. Define $\sR(\pi) := \sup_{\vct\nu\in\{0,1\}^J}\mathbb E_{\vct\nu}^\pi[\sum_{t=1}^T f(p^*(\vct x_t),\vct x_t)-f(p_t,\vct x_t)]$.
Then there exists an $2\varepsilon$-LDP personalized pricing policy $\pi'=\{Q_t',A_t\}_{t=1}^T$ with the same regret performance as $\pi$,
and a classifier $\psi:\{z_1',\cdots,z_T'\}\mapsto\hat{\vct\nu}\in\{0,1\}^J$ such that
\begin{equation}
\sup_{\vct\nu\in\{0,1\}^J}\mathbb E_{\vct\nu}^{\pi'}\left[\sum_{j=1}^J\vct 1\{\hat\nu_j\neq\nu_j\}\bigg|\hat{\vct\nu}=\psi(z_1',\cdots,z_T'), z_t'\sim Q_t'(\cdot|s_t,z_1',\cdots,z_{t-1}')\right] \leq \frac{\bar K_1d^2J}{h^2T}\sR(\pi),
\end{equation}
where $\bar K_1$ is a universal numerical constant.
\label{lem:reduction-classification}
\end{lemma}

To simplify notations, in the rest of the proof we shall drop the prime superscripts and simply denote $\pi,z_t$ as $\pi',z_t'$ 
associated with the augmented $2\varepsilon$-LDP personalized pricing policy constructed in Lemma \ref{lem:reduction-classification}.
Essentially, Lemma \ref{lem:reduction-classification} shows that if there is a low-regret LDP pricing policy,
then there is an LDP classifier $\psi$ that achieves small classification errors on $\vct\nu$ as well.
Naturally, the next step is to use information-theoretical tools to show a lower bound on the classification error attainable by \emph{any}
LDP classifiers, which would then lead to a lower bound on the cumulative regret of any LDP personalized pricing policy $\pi$.

To derive a lower bound on the classification error of $\psi$, we shall use the celebrated \emph{Assuard's method} in the mathematical statistics
and information theory literature \citep{assouad1983deux,yu1997assouad}.
To state the method we need some additional notations.
For each $\vct\nu\in\{0,1\}^J$, let $M_{\vct\nu}^\pi$ be the distribution of the intermediate quantities $\{z_1,\cdots,z_T\}$
under model $f_{\vct\nu}$ and personalized pricing policy $\pi$.
For each $j\in[J]$, define
\begin{equation}
M_{+j}^\pi := \frac{1}{2^{J-1}}\sum_{\vct\nu:\vct\nu_j=1}M_{\vct\nu}^\pi, \;\;\;\;\;\;M_{-j}^\pi := \frac{1}{2^{J-1}}\sum_{\vct\nu:\vct\nu_j=0}M_{\vct\nu}^\pi
\label{eq:defn-Mj}
\end{equation}
as the mixture distribution by fixing the $j$th bit of $\vct\nu$ to either one or zero.
For two distributions $P,Q$ let $\|P-Q\|_{\tv}=\frac{1}{2}\int|\ud P-\ud Q|$ be the total variation distance between $P$ and $Q$,
and $D_{\kl}^{\sy}(P,Q)=D_{\kl}(P\|Q)+D_{\kl}(Q\|P)=\int (\ud P-\ud Q)\ln(\ud P/\ud Q)$ be the symmetric Kullback-Leibler (KL) divergence
between $P$ and $Q$.
Using the Assouad's Lemma (see, e.g., Lemma 1 and Eq.~(29) of \citealt{duchi2018minimax}) and Pinsker's inequality, we have
\begin{align}
\inf_{\pi,\psi}\sup_{\vct\nu\in\{0,1\}^d}\mathbb E_{\vct\nu}^{\pi,\psi}\left[\sum_{j=1}^J\hat\nu_j\neq\nu_j\right]
&\geq \frac{1}{2}\sum_{j=1}^J\left(1-\|M_{+j}^\pi-M_{-j}^\pi\|_{\tv}\right)
\geq \frac{1}{2}\sum_{j=1}^J\left(1-\sqrt{\frac{1}{4}D_{\kl}^\sy(M_{+j}^\pi,M_{-j}^\pi)}\right)\nonumber\\
&\geq \frac{J}{2}\left(1 - \sqrt{\frac{1}{4J}\sum_{j=1}^JD_{\kl}^\sy(M_{+j}^\pi,M_{-j}^\pi)}\right),
\label{eq:assouad}
\end{align}
where the last inequality holds by Cauchy-Schwarz inequality.
The question of upper bounding the symmetric KL-divergence between $M_{+j}^\pi$ and $M_{-j}^\pi$,
crucial to lower bounding the right-hand side of Eq.~(\ref{eq:assouad}), is addressed in the next section.

\subsection{Strong Data Processing Inequality}

The main objective of this section is to provide technical tools to upper bound $D_{\kl}^\sy(M_{+j}^\pi,M_{-j}^\pi)$.
We first define some notations. Let $\vct z_{<t}=(z_1,\cdots,z_{<t})$, and $M_{\pm j,<t}^\pi$ be the marginal distributions of $\vct z_{<t}$
under $P_{\pm j}^\pi$, where $P_{+j}^\pi=\frac{1}{2^{J-1}}\sum_{\vct\nu:\vct\nu_j=1}P_{\vct\nu}^\pi$ and $P_{-j}^\pi=\frac{1}{2^{J-1}}\sum_{\vct\nu:\vct\nu_j=0}P_{\vct\nu}^\pi$.
Let $P_{\pm j,t}^\pi$ be the distribution of $s_t$, which is measurable conditioned on $\vct z_{<t}$ since the price $p_t$ being offered by $A_t$
can only depends on $\vct x_t$ and $\vct z_{<t}$ (see Eq.~(\ref{eq:defn-At})).
We establish the following lemma:
\begin{lemma}
Let $\pi$ be a personalized pricing policy that satisfies $2\varepsilon$-LDP. Then
\begin{align*}
\sum_{j=1}^J&D_{\kl}^\sy(M_{+j}^\pi,M_{-j}^\pi)\\
&\leq 2(e^{2\varepsilon}-1)^2\sum_{t=1}^T\sup_{\|\gamma\|_{\infty}\leq 1}\sum_{j=1}^J\int_{\mathcal Z^{t-1}}\left|\int_{\mathcal S}\gamma(s_t,\vct z_{<t})\big[\ud P_{+j,t}^\pi(s_t|\vct z_{<t})-\ud P_{-j,t}^\pi(s_t|\vct z_{<t})\big]\right|^2\ud \bar M_{<t}^\pi(\vct z_{<t})\\
&\leq  8(e^{2\varepsilon}-1)^2\sum_{t=1}^T\sum_{j=1}^J\mathbb E_{\vct z_{<t}\sim\bar M_{<t}^\pi(\vct z_{<t})}\left[\|P_{+j,t}^\pi(\cdot|\vct z_{<t})-P_{-j,t}^\pi(\cdot|\vct z_{<t})\|_{\tv}^2\right],
\end{align*}
where $\bar M_{<t}^\pi = \frac{1}{2^J}\sum_{\vct\nu\in\{0,1\}^J}M_{\vct\nu,<t}^\pi$, and $\gamma(s_t,\vct z_{<t})$ is any arbitrary function of $s_t,\vct z_{<t}$ with $|\gamma(s_t,\vct z_{<t})|\leq 1$ for any $s_t,\vct z_{<t}$.
\label{lem:sddi-ldp}
\end{lemma}

Lemma \ref{lem:sddi-ldp} is in principle similar to Theorem 3 of \citep{duchi2018minimax}, with one significant difference:
for all the results in the work of \cite{duchi2018minimax}, the sensitive data $\{s_t\}_{t=1}^T$ are distributed independently and identically
with respect to an unknown distribution which does not depend on the estimator or algorithm used (i.e., the classical statistical estimation setting).
In contrast, for our personalized pricing problem the distribution of the sensitive information $\{s_t\}_{t=1}^T$ depends on both the underlying
demand model $f_{\vct\nu}$ and the (locally private) pricing policy $\pi$, as shown in the $P_{+j,t}^\pi$ and $P_{-j,t}^\pi$ measures
that are measurable conditioned on $\vct z_{<t}$.
This leads to a more sophisticated upper bound on the symmetric KL-divergence between $M_{+j}^\pi$ and $M_{-j}^\pi$ measures as shown in Lemma \ref{lem:sddi-ldp}.

Note that $P_{\pm j,t}^\pi(\vct x_t,y_t,p_t|\vct z_{<t})=\chi(\vct x_t)A_t(p_t|\vct x_t,\vct z_{<t})\Pr_{\pm j}(y_t|p_t,\vct x_t)$,
where $\chi$ is the PDF of $P_X$ being the uniform distribution on $[0,1]^d$,
$A_t(\cdot|\vct x_t,\vct z_{<t})$ is the pricing distribution of $\pi$ which is measurable conditioned on $\vct x_t$ and $\vct z_{<t}$ by its
definition in Eq.~(\ref{eq:defn-At}), and $\Pr_{+j}(y_t|p_t,\vct x_t)=\frac{1}{2^{J-1}}\sum_{\vct\nu:\vct\nu_j=1}\Pr(y_t|f_{\vct\nu}(p_t,\vct x_t))$
and $\Pr_{-j}(y_t|p_t,\vct x_t) = \frac{1}{2^{J-1}}\sum_{\vct\nu:\vct\nu_j=0}\Pr(y_t|f_{\vct\nu}(p_t,\vct x_t))$.
We then have the following lemma upper bounding the total variation between $P_{\pm j,t}^\pi(\cdot|\vct z_{<t})$:
\begin{lemma}
For every $j\in[J]$, $t\in[T]$ and $\vct z_{<t}\in\mathcal Z^{t-1}$, it holds that
$$
\|P_{+j,t}^\pi(\cdot|\vct z_{<t}) - P_{-j,t}^\pi(\cdot|\vct z_{<t})\|_\tv^2 \leq\frac{dh^2}{16J^2}\mathbb E_{\vct x\sim U(B_j)}\mathbb E_{p\sim A_t(\cdot|\vct x,\vct z_{<t})}[(p-p_0)^2],
$$
where $p_0=2/3$ and $U(B_j)$ is the uniform distribution on $B_j$.
\label{lem:ppm-ldp}
\end{lemma}

Combining Lemmas \ref{lem:sddi-ldp} and \ref{lem:ppm-ldp}, and noting that $\Pr[\vct x\in B_j]=\int_{B_j}\ud P_X(\vct x)=1/J$ for all $j\in[J]$, we arrive at the following corollary:
\begin{corollary}
Let $\pi$ be a personalized pricing policy that satisfies $2\varepsilon$-LDP, and $p_0=2/3$. Then
$$
\sum_{j=1}^JD_{\kl}^\sy(M_{+j}^\pi,M_{-j}^\pi)\leq \frac{dh^2(e^{2\varepsilon}-1)^2}{2J}\times \sum_{t=1}^T \mathbb E\left[(p-p_0)^2\bigg|p\sim A_t(\cdot|\vct x,\vct z_{<t}), \vct x\sim P_X, \vct z_{<t}\sim \bar M_{<t}^\pi\right].
$$
\label{cor:sddi-ldp}
\end{corollary}

\subsection{Completing the Proof of Theorem \ref{thm:regret-lower-bound}}

We first establish the following technical lemma showing that, if a personalized pricing policy $\pi$ has small regret
then it must produce prices that are close to $p_0=2/3$.
\begin{lemma}
Let $\pi$ be a personalized pricing policy, $p_0=2/3$ and $\phi(x,\delta) := x\vct 1\{|x|\geq\delta\}$ be the hard-thresholding operator.
Then 
$$
\sR(\pi)\geq \frac{1}{4}\sum_{t=1}^T\mathbb E\left[\phi(|p_t-p_0|^2,h^2)\bigg|p_t\sim A_t(\cdot|\vct x_t,\vct z_{<t}),\vct x_t\sim P_X,\vct z_{<t}\sim\bar M_{<t}^\pi\right].
$$
\label{lem:regret-self-bound}
\end{lemma}

We are now ready to prove Theorem \ref{thm:regret-lower-bound}.
\begin{proof}{Proof of Theorem \ref{thm:regret-lower-bound}.}
First, if $\sR(\pi)\geq \varepsilon^{-2/(d+2)}T^{(d+1)/(d+2)}$ we have already proved Theorem \ref{thm:regret-lower-bound}.
Hence, in the rest of the proof we assume that $\sR(\pi)\leq \varepsilon^{-2/(d+2)}T^{(d+1)/(d+2)}=:R$.
Note that $\phi(|p-p_0|^2,h^2) \geq (p-p_0)^2 - h^2$, and $e^{2\varepsilon}-1\leq 8\varepsilon$ for all $\varepsilon\in(0,1]$.
Note also that $J=h^{-d}$ by definition.
Corollary \ref{cor:sddi-ldp} and Lemma \ref{lem:regret-self-bound} together yield
\begin{equation}
\sum_{j=1}^JD_{\kl}^\sy(M_{+j}^\pi,M_{-j}^\pi)\leq 128d h^{2+d}\varepsilon^2(h^2T + R).
\label{eq:proof-lower-bound-1}
\end{equation}

Set hypercube size $h$ as
\begin{equation}
h := (8^{-1}\varepsilon\sqrt{dT})^{-1/(d+2)}.
\label{eq:proof-lower-bound-2}
\end{equation}
It is then easy to verify that $\frac{1}{4J}\sum_{j=1}^JD_{\kl}^\sy(M_{+j}^\pi,M_{-j}^\pi)\leq \frac{1}{2}$. Subsequently, Eq.~(\ref{eq:assouad}) yields
\begin{equation}
\inf_{\pi,\psi}\sup_{\vct\nu\in\{0,1\}^d}\mathbb E_{\vct\nu}^{\pi,\psi}\left[\sum_{j=1}^J\hat\nu_j\neq\nu_j\right]
\geq \frac{J}{2}\left(1-\frac{1}{\sqrt{2}}\right) \geq 0.15J.
\label{eq:proof-lower-bound-3}
\end{equation}
Combining Eq.~(\ref{eq:proof-lower-bound-3}) with Lemma \ref{lem:reduction-classification},
we obtain (in the $\Omega(\cdot)$ notation below we omit only universal numerical constants)
\begin{align*}
\sR(\pi) &\geq \frac{h^2 T}{\bar K_1 d^2 J}\times 0.15J \geq \Omega\left(\frac{h^2 T}{d^2}\right)\geq \Omega\left(\frac{\varepsilon^{-2/(d+2)}T^{(d+1)/(d+2)}}{d^{7/3}}\right),
\end{align*}
which proves the regret lower bound in Theorem \ref{thm:regret-lower-bound}.

Finally, we verify that the conditions $T\geq 6.2\sqrt{\ln(2T)}$ and $\sqrt{T}\geq 66\varepsilon^{-1}J\sqrt{\ln(2T)}$ are satisfied (which is required in Lemma \ref{lem:reduction-classification}).
To satisfy $T\geq 6.2\sqrt{\ln(2T)}$ we only need $T$ to be sufficiently large (larger than some absolute constant). 
The second condition, on the other hand, translates to $\sqrt{T/\ln(2T)}\geq\Omega(\varepsilon^{-1}h^{-d}) = \Omega(\varepsilon^{-1}(\varepsilon\sqrt{dT})^{d/(d+2)})=\Omega(\varepsilon^{-1/(d+2)}d^{d/2(d+2)}\sqrt{T}^{d/(d+2)})$.
Hence, with the asymptotic scaling of $T=\Omega(\varepsilon^{-1}d^{d/2+1}\ln(d/\varepsilon))$,
the condition $\sqrt{T}\geq 66\varepsilon^{-1}J\sqrt{\ln(2T)}$ is satisfied for sufficiently large $T$. $\square$
\end{proof}

\section{Numerical Experiments}\label{sec:num_exp}
In this section, we conduct some illustrative numerical experiments to show the effect of central and local differential privacy on the regret. To do that, we assume the dimension $d=2$ and the demand $y_t(p)$ is a linear demand, i.e., $y_t(p)=\theta_{0}+\theta_1 x_{t,1}+\theta_2 x_{t,2}+\theta_3 p+\nu_t$, where $\nu_t$ is an independent zero-mean noise. Since this example is for illustrative purpose, the value of $\vct\theta$ is taken as $(0.4,0.6,0.6,-0.2)$, $\nu_t\in [-0.1,0.1]$ is a uniform distribution, and $p\in [0.5,4.5]$. Customer's data $\vct x_t$ is taken uniformly from $[0,1]^2$. For the input parameters of the two algorithms (CPPQ and LPPQ), we simply fix $c_1=0.001\sqrt{\ln(T)},c_1'=0.01c_2,c_2=\varepsilon^{-1}\ln^2(T)$, and $\kappa_1 = 0.001\sqrt{\ln(T)},\kappa_2 = 0.1\ln(T)$ for LPPQ. Both algorithms have 30 independent runs with $T\in\{500,2500,12500,62500\}$ and $\varepsilon\in\{0.01,0.1,1,10\}$, and for each $T$ and $\varepsilon$, the average is taken as the output. 

%
%
%

\begin{table}[!h]
	\vspace{.25in}
	\centering
	\begin{tabular}{l|cccc}
		\hline
		\hline
		&  $T=500$ & $T=2500$ & $T=12500$ & $T=62500$  \\ \hline
		Non-Private     & 15.79       & 7.40       & 3.33       & 1.76     \\
		$\varepsilon=10$       & 26.77       & 20.68       & 12.65      & 8.68       \\
		$\varepsilon=1$        & 34.61       & 31.48       & 25.89       & 21.04       \\
		$\varepsilon=0.1$    & 34.81      & 33.06      & 29.89       & 26.72 \\
		$\varepsilon=0.01$   & 34.70       & 33.63       & 30.51       & 27.21   \\
		\hline
		\hline
	\end{tabular}
	\caption{Percentage regret ($\%$) for CPPQ.}
	\label{tab:CPPQ}
\end{table}

\begin{table}[!h]
	\vspace{.25in}
	\centering
	\begin{tabular}{l|cccc}
		\hline
		\hline
		&  $T=500$ & $T=2500$ & $T=12500$ & $T=62500$  \\ \hline
		Non-Private     & 15.79       & 7.40       & 3.33       & 1.76     \\
		$\varepsilon=10$     & 21.82       & 17.53      & 15.50       & 13.27       \\
		$\varepsilon=1$      & 20.81       & 17.40       & 15.73       & 14.29       \\
		$\varepsilon=0.1$     & 22.89      & 17.66      & 15.95       & 14.80  \\
		$\varepsilon=0.01$     & 22.53       & 20.70       & 17.20       & 16.74   \\
		\hline
		\hline
	\end{tabular}
	\caption{Percentage regret ($\%$) for LPPQ.}
	\label{tab:LPPQ}
\end{table}

\begin{figure}
	\centering
	\includegraphics[width=0.5\textwidth]{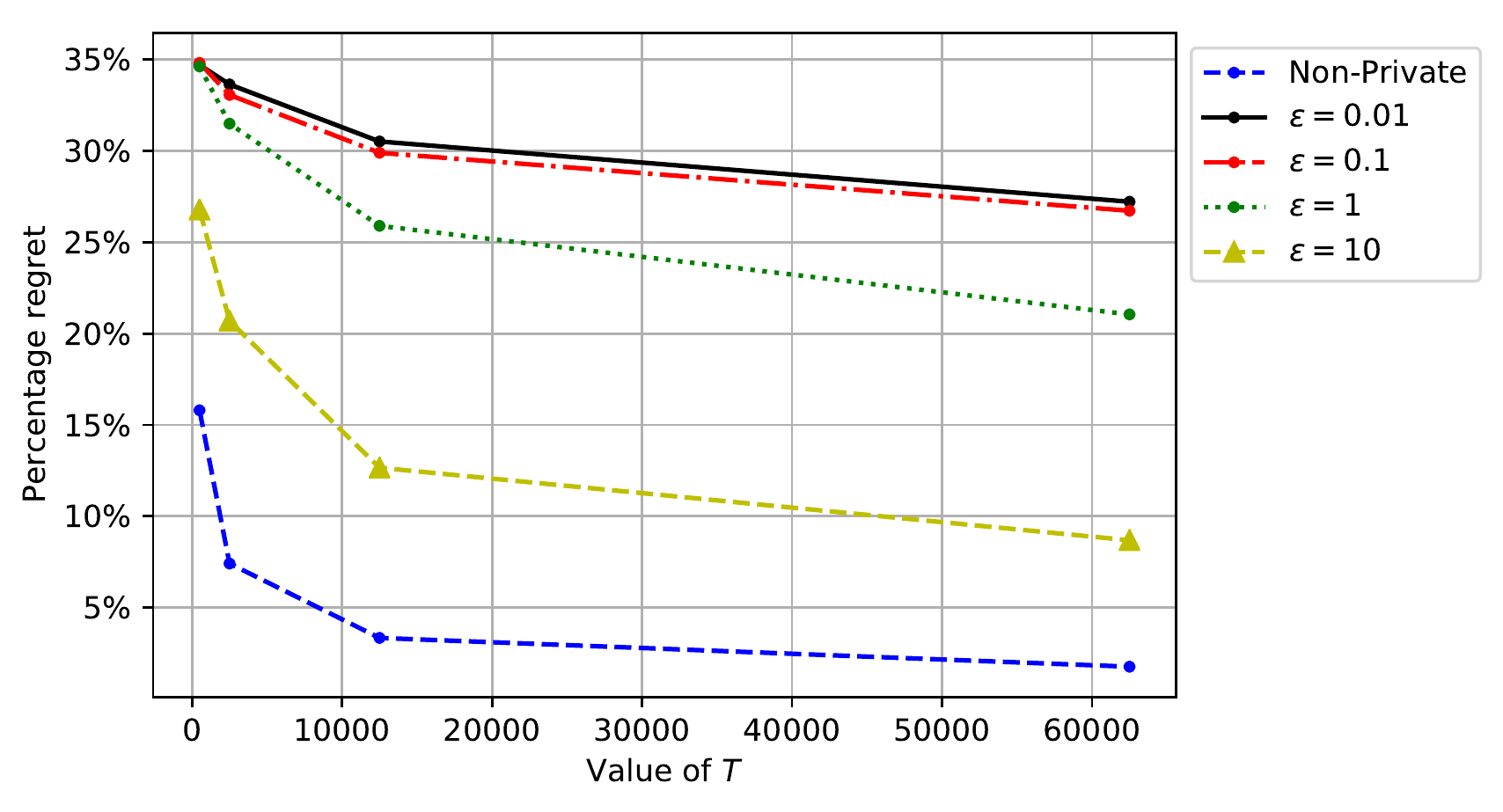}
	~
	\includegraphics[width=0.5\textwidth]{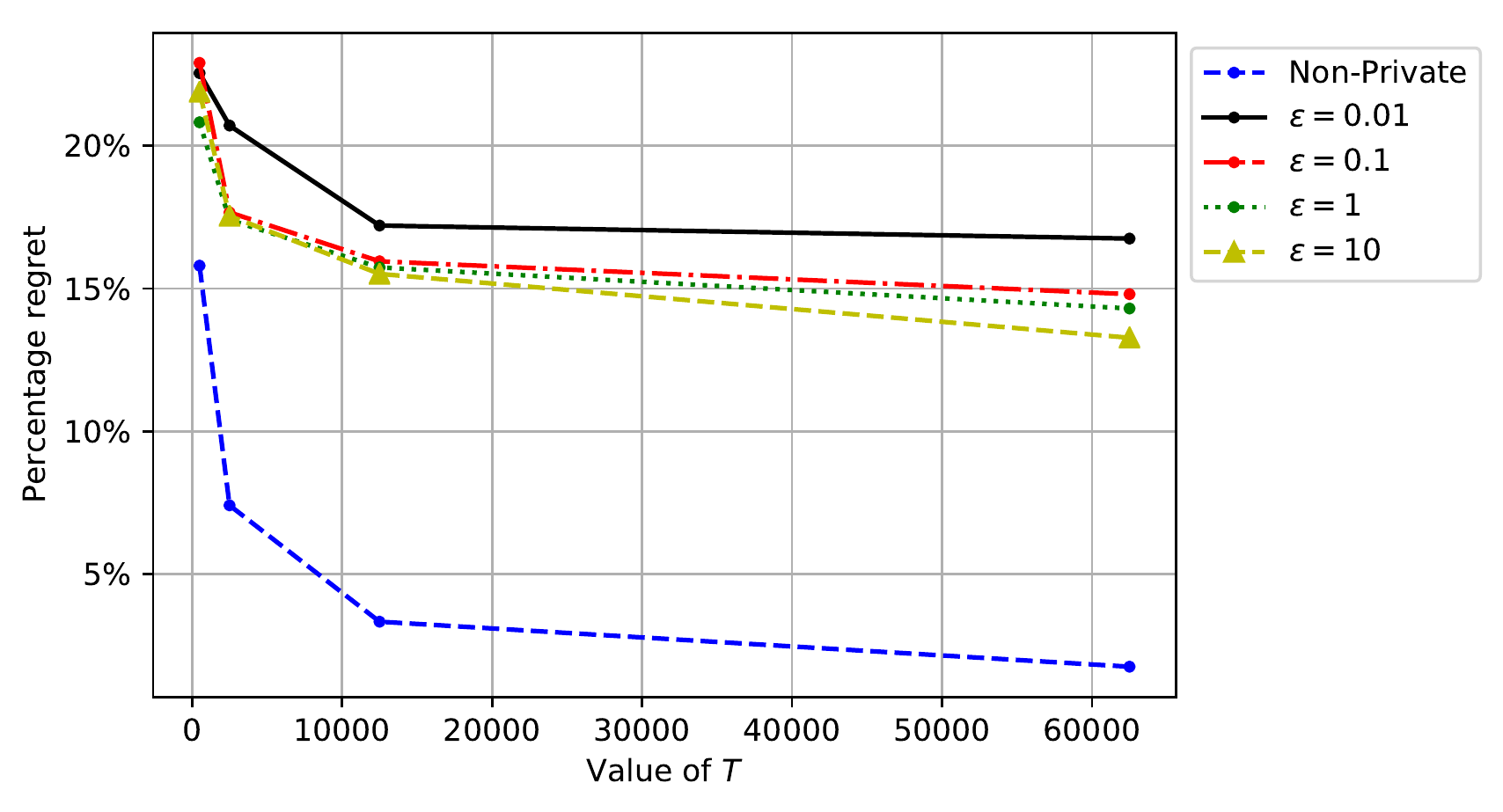}
	\caption{Percentage regret with respect to $T$ of CPPQ (the left panel) and LPPQ (the right panel). In both cases, the percentage regrets converge to zero as $T$ grows, and a larger $\varepsilon$ (i.e., less privacy protection)  leads to a smaller percentage regret.}
	\label{fig:cppq_lppq_T}
\end{figure}

\begin{figure}
	\centering
	\includegraphics[width=0.6\textwidth]{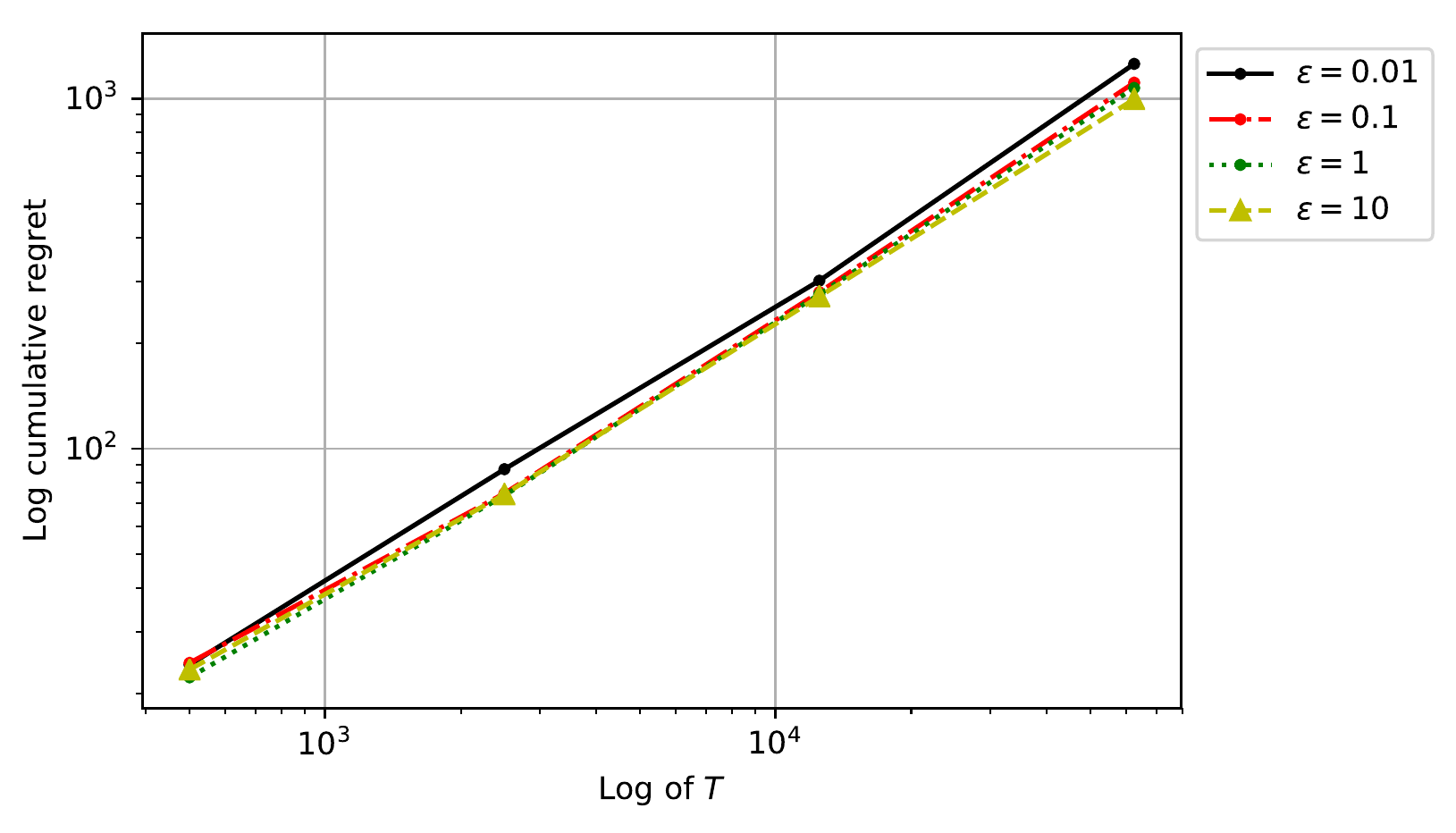}
	\caption{Log-log plot of $\sR_T(f,\pi)/\ln(T)$ for LPPQ for all $\varepsilon$ and $T$. The slopes of fitted lines with $\varepsilon\in\{0.01,0.1,1,1\}$ are $\{0.79,0.77,0.77,0.75\}$ respectively.}
	\label{fig:loglog_LPPQ}
\end{figure}



Results of the two algorithms are summarized in Table \ref{tab:CPPQ} (for CPPQ), Table \ref{tab:LPPQ} (for LPPQ), and Figure \ref{fig:cppq_lppq_T}. To compare the performance of the algorithms across different $T$ and $\varepsilon$, we compute the percentage regret, which is defined as $\sR_T(f,\pi)/\sum_{t=1}^{T}f(p^*(\vct x_t),\vct x_t)$. For the non-private benchmark, we adopt the non-private version of CPPQ (i.e., without adding noise) because it estimates the revenue more directly.
According to these results, we can see that larger $\varepsilon$ in general leads to better regrets, which are closer to the one with no privacy. Moreover, it can be observed that the performance of CPPQ is quite sensitive to $\varepsilon$ (especially when $\varepsilon$ is relatively large). 
This is in line with Theorem \ref{thm:cppt-regret} as the the impact of $\varepsilon$ on the regret is by the factor of $\varepsilon^{-1}$.
For LPPQ, our observation is similar. That is, higher privacy leads to higher percentage regret. However, it shall be noted that the difference of percentage regret with respect to privacy parameter $\varepsilon$ is not very significant. This observation is consistent with our theoretical results in Theorem \ref{thm:regret-lppt}, which shows that the dependency of regret on $\varepsilon$ is $\varepsilon^{-2/(d+2)}$. Moreover, in Figure \ref{fig:loglog_LPPQ}, we plot the log-log scale of cumulative regret $\sR_T(f,\pi)/\ln(T)$ (it is divided by $\ln(T)$ because the regret upper bound has $\ln(T)$ factor as shown in Theorem \ref{thm:regret-lppt}) of LPPQ with all values of $\varepsilon$. The slopes of the fitted lines are $\{0.79,0.77,0.77,0.75\}$ with respect to $\varepsilon\in\{0.01,0.1,1,10\}$, which are quite close to $(d+1)/(d+2)=0.75$ when $d=2$. Thus, our numerical results verify the regret upper bound $\tilde O(\varepsilon^{-2/(d+2)}T^{(d+1)/(d+2)})$ (with respect to $T$) of LPPQ.  

In the end, comparing the result of CPPQ and LPPQ, we see that CPPQ does not necessarily have better performance than LPPQ, even though $\varepsilon$-CDP is weaker than $\varepsilon$-LDP in many cases. One reason is that in CPPQ, the Laplace noise has a scale of $2(L+1)/\varepsilon$ as opposed to $2/\varepsilon$ in LPPQ. As a result, for small $\varepsilon$, Lap($2(L+1)/\varepsilon$) can be quite significant especially in a relatively short horizon. This result is actually not surprising from the theoretical performance of CPPQ (i.e., $\tilde O(\varepsilon^{-1}T^{d/(d+4)})$ for the part with $\varepsilon$) versus LPPQ (i.e., $\tilde O(\varepsilon^{-2/(d+2)}T^{(d+1)/(d+2)})$). That is, when $\varepsilon$ is small while $T$ is not very large, it is very likely $\varepsilon^{-1}T^{d/(d+4)}> \varepsilon^{-2/(d+2)}T^{(d+1)/(d+2)}$ (e.g., $\varepsilon<o(T^{-5/12})$ when $d=2$).

\section{Conclusion}\label{sec:conclusion}
This paper studies the online personalized pricing problem with nonparametric demand and data privacy (to the best of our knowledge, our paper is the first result on this problem). That is, over a finite time horizon, the platform decides a price for each arriving customer based on  her personal data in order to maximize the cumulative revenue and protect customer's privacy. 
Two definitions of data privacy have been investigated: $\varepsilon$-CDP and $\varepsilon$-LDP, each of which depends on a parameter $\varepsilon>0$ (i.e., smaller $\varepsilon$ means higher security). 
Two algorithms are developed in this paper: CPPQ for $\varepsilon$-CDP and LPPQ for $\varepsilon$-LDP. Both algorithms are based on the idea of splitting the domain of customer's data into hypercubes, and applying a novel quadrisection search of optimal price in each hypercube. To satisfy $\varepsilon$-CDP, CPPQ uses a tree-based aggregation method to estimate the reward of the tested prices in each hypercube. Results show that this algorithm has regret at least $\tilde O(T^{(d+2)/(d+4)}+\varepsilon^{-1}T^{d/(d+4)})$ (Theorem \ref{thm:cppt-regret}), where the first term matches the near-optimal regret of $\tilde O(T^{(d+2)/(d+4)})$ without $\varepsilon$-CDP \citep{chen2021nonparametric}. On the other hand, LPPQ protects the privacy by adding noise to each data sample, and it achieves a regret of $\tilde O(\varepsilon^{-2/(d+2)}T^{(d+1)/(d+2)})$ (Theorem \ref{thm:regret-lppt}). Moreover, this regret is proved to be near-optimal in Theorem \ref{thm:regret-lower-bound}, which shows that any algorithm satisfying $\varepsilon$-LDP has the regret at least $\Omega(\varepsilon^{-2/(d+2)}T^{(d+1)/(d+2)})$. 

For potential future research, one direction is to consider data privacy protection with nonparametric model in other operations management problems (e.g., healthcare, assortment selection, ranking). Second, one may consider other techniques of protecting differential privacy in personalized pricing problem. For instance, the platform may add noise directly to historical customers' data $\vct x_t$, and a potential technique of demand learning in this scenario is the so-called nonparametric regression with error in variables (see, e.g., \citealt{fan1993nonparametric}). Developing near-optimal online learning algorithms using this technique is an interesting open problem.

\bibliography{refs}
\bibliographystyle{apa-good}

\ECSwitch

\ECDisclaimer

\newcommand{\mD}{\mathcal D}


\ECHead{Supplementary material}

\section{Proof of Theorem \ref{thm:cppt-regret}}

We first establish the $\varepsilon$-CDP property of Algorithm \ref{alg:cppt}.
For each $t$ and $k=k_t$ define $\vct u_t:=(u_{j,t,k_t})_{j\in J}\in\mathbb R^J$ and $\vct v_t := (v_{j,t,k_t})_{j\in J}$.
It is clear that both $\vct u_t,\vct v_t$ are one-hot vectors (i.e., having zero values except for one entry $j=j_t$)
and hence their $\ell_1$ sensitivity is upper bounded by 2.
Note also that each time $\textsc{TreeBasedAggregation}$ is invoked, the resulting $\hat S_t$ consists of at most $L+1=\lfloor\log_2 T\rfloor+1$
partial sums. Therefore, by simple composition of Laplace mechanisms we have that all $\{r_{j,k}(t),\mu_{j,k}(t)\}_{t=1}^T$ are $\varepsilon$-differentially
private.
Since other parts of Algorithm \ref{alg:cppt} (e.g., the updates of $\vct\rho_j$ and $\varsigma_j$) no longer access the sensitive data $\{s_t=(\vct x_t,y_t,p_t)\}_{t=1}^T$ except for the offering price $p_t$, we conclude that Algorithm \ref{alg:cppt} satisfies $\varepsilon$-CDP thanks to the closedness-to-post-processing
property of differential privacy.

In the rest of the proof we upper bound the regret of Theorem \ref{thm:cppt-regret}.
For simplicity we will use the notation $\const$ to denote any universal numerical constant that does not depend on any problem parameters.
For $j\in[J]$, $t\in[T]$ and $k\in[5]$ recall the definitions that $u_{j,k,t}=\vct 1\{j=j_t\wedge k=k_t\}p_ty_t$ and $v_{j,k,t}=\vct 1\{j=j_t\wedge k=k_t\}$.
We first state a technical lemma that upper bounds (with high probability) the deviation of $r_{j,k}(t),\mu_{j,k}(t)$ from $\sum_{s\leq t}u_{s,j,k}$ and $\sum_{s\leq t}v_{s,j,k}$, respectively.
\begin{lemma}
With probability $1-O(T^{-1})$ the following hold uniformly over all $t\in[T]$, $j\in[J]$ and $k\in[5]$:
$$
\max\left\{\left|r_{j,k}(t) - \sum_{s\leq t}u_{s,j,k}\right|, \left|\mu_{j,k}(t)-\sum_{s\leq t}v_{s,j,k}\right|\right\}\leq 19\varepsilon^{-1}\ln^2(2T^3).
$$
\label{lem:cppt-aggregation}
\end{lemma}
\begin{proof}{Proof of Lemma \ref{lem:cppt-aggregation}.}
We focus on the $|r_{j,k}(t)-\sum_{s\leq t}u_{s,j,k}|$ term only as the proof for the other term is exactly the same.
Let $t=\sum_{\ell=0}^Lb_\ell(t)2^\ell$ where $L=\lfloor\log_2 T\rfloor$ and $b_\ell(t)\in\{0,1\}$ be the binary expression of $t$.
By the definition of the \textsc{TreeBasedAggregation}, we have that
\begin{equation}
r_{j,k}(t) - \sum_{s\leq t}u_{s,j,k} = \sum_{\ell=0}^L b_\ell(t)w_\ell,
\label{eq:proof-aggregation-1}
\end{equation}
where $\{w_\ell\}_{\ell=0}^L\overset{i.i.d.}{\sim}\Lap(2/\varepsilon')$ with $\varepsilon'=\varepsilon/(L+1)$.
Invoke the concentration inequality for sums of i.i.d.~centered Laplace random variables \citep[Corollary 2.9]{chan2011private}.
We have with probability $1-O(T^{-3})$ that
\begin{align}
\left| \sum_{\ell=0}^L b_\ell(t)w_\ell\right|
&\leq 2\varepsilon^{-1}(L+1)\left(\sqrt{L+1}+\sqrt{\ln(2T^3)}\right)\sqrt{8\ln(2T^3)}\nonumber\\
&\leq 19\varepsilon^{-1}\ln^2(2T^3).\label{eq:proof-aggregation-2}
\end{align}
Combine Eqs.~(\ref{eq:proof-aggregation-1},\ref{eq:proof-aggregation-2}) and apply a union bound over all $t\in[T]$, $j\in[J]$ and $k\in[5]$.
Lemma \ref{lem:cppt-aggregation} is thus proved. $\square$
\end{proof}

In the rest of the proof, define $n_j=t-\varsigma_j$, $\hat r_{jk}=r_{jk}(t)-r_{jk}(\varsigma_j)$ and $\hat \mu_{jk}=\mu_{jk}(t)-\mu_{jk}(t)(\varsigma_j)$.
Define $\tilde r_{jk} := \sum_{\tau=\varsigma_j+1}^t u_{j,\tau,k}$ and $\tilde\mu_{jk}:=\sum_{\tau=\varsigma_j+1}^t v_{j,\tau,k}$.
Lemma \ref{lem:cppt-aggregation} implies that with probability $1-O(T^{-1})$, 
\begin{equation}
\max\{|\tilde r_{jk}-\hat r_{jk}|,|\tilde\mu_{jk}-\hat \mu_{jk}|\}\leq 38\varepsilon^{-1}\ln^2(2T^3).
\label{eq:proof-cppt-1}
\end{equation}
On the other hand, observe that both $\tilde\mu_{jk}$ and $\tilde r_{jk}$ are sums of i.i.d.~random variables.
The following technical lemma then upper bounds the concentration of $\tilde\mu_{jk}$ and $\tilde r_{jk}$ towards their expected values.
\begin{lemma}
For each hypercube $B_j$ define $\bar\chi(B_j) := J\times \Pr[\vct x\in B_j] \in [0,C_X]$.
Let $n_j=t-\varsigma_j$.
The following holds with probability $1-O(T^{-1})$ uniformly over all $t\in[T]$, $j\in[J]$ and $k\in[5]$:
\begin{equation}
\max\left\{\left|\frac{\tilde r_{jk}}{n_j} - \frac{\bar\chi(B_j)h^d}{5n_j} f_{B_j}(\rho_{jk}) \right|, \left|\frac{\tilde\mu_{jk}}{n_j} - \frac{\bar\chi(B_j)h^d}{5n_j}\right|\right\} \leq  \sqrt{\frac{C_Xh^d\ln(2T^3)}{n_j}} + \frac{1.5\ln(2T^3)}{n_j}.\label{eq:cppt-concentration-1}
\end{equation}
Furthermore,
\begin{equation}
\left|\frac{\tilde r_{jk}}{\tilde\mu_{jk}} - f_{B_j}(\rho_{jk})\right| \leq \sqrt{\frac{\ln(2T^3)}{2\tilde\mu_{jk}}}.
\label{eq:cppt-concentration-2}
\end{equation}
\label{lem:cppt-concentration}
\end{lemma}
\begin{proof}{Proof of Lemma \ref{lem:cppt-concentration}.}
For Eq.~(\ref{eq:cppt-concentration-1})
we will prove the concentration inequality involving $\tilde\mu_{jk}$ only, as the other inequality can be proved using the exact same argument.
Note that for each $j\in[J]$, $s\in[\varsigma_j+1,t]$ with $s\equiv k\mod 5$, we have that $v_{j,s,k}\in\{0,1\}$ almost surely
and furthermore $\Pr[v_{s,j,k}=1] = \Pr[\vct x_s\in B_j, k\equiv s\mod 5] = \bar\chi(B_j)h^d/5$. By Bernstein's inequality \citep{bennett1962probability}, for any particular $t$, $j$ and $k$
it holds with probability $1-O(T^{-3})$ that
$$
\left|\frac{\tilde\mu_{jk}}{n_j} - \frac{\bar\chi(B_j)h^d}{5n_j}\right|\leq 2\sqrt{\frac{\bar\chi(B_j)h^d\ln(2T^3)}{5 n_j}}+\frac{1.5\ln(2T^3)}{n_j}
\leq \sqrt{\frac{C_Xh^d\ln(2T^3)}{ n_j}} + \frac{1.5\ln(2T^3)}{n_j}.
$$
Apply the union bound over all $t\in[T]$, $j\in[J]$ and $k\in[5]$. We complete the proof of Eq.~(\ref{eq:cppt-concentration-1}). 

We next prove Eq.~(\ref{eq:cppt-concentration-2}). Note that $\tilde r_{jk}$ is a sum of $\tilde\mu_{jk}$ i.i.d.~random variables each with expectation $f_{B_j}(\rho_{jk})$ and supported on $[0,1]$ almost surely. 
By Hoeffding's inequality \citep{hoeffding1963probability}, it holds with probability $1-O(T^{-3}$) that 
$$
\left|\frac{\tilde r_{jk}}{\tilde\mu_{jk}} - f_{B_j}(\rho_{jk})\right|\leq \sqrt{\frac{\ln(2T^3)}{2\tilde\mu_{jk}}}.
$$
Apply the union bound over all $t\in[T]$, $j\in[J]$ and $k\in[5]$. We complete the proof of Eq.~(\ref{eq:cppt-concentration-2}). 
$\square$
\end{proof}

Eq.~(\ref{eq:proof-cppt-1}) and Lemma \ref{lem:cppt-concentration} together yield the following lemma:
\begin{lemma}
With probability $1-O(T^{-1})$ the following holds uniformly over all $t\in[T]$, $j\in[J]$, $k\in[4]$ that satisfies $\hat\mu_{jk}\geq c_2=76\varepsilon^{-1}\ln^2(2T^3)$:
$2\hat\mu_{jk}\geq\tilde\mu_{jk}\geq \hat\mu_{jk}/2$, and furthermore
$$
\left|\frac{\hat r_{jk}}{\hat\mu_{jk}} - f_{B_j}(\rho_{jk})\right|\leq \frac{2c_2}{\tilde\mu_{jk}} + \sqrt{\frac{\ln(2T^3)}{2\tilde\mu_{jk}}}
\leq \frac{4c_2}{\hat\mu_{jk}} + \sqrt{\frac{\ln(2T^3)}{\hat\mu_{jk}}}.
$$
\label{lem:rhojk-concentration}
\end{lemma}
\begin{proof}{Proof of Lemma \ref{lem:rhojk-concentration}.}
Fix $t\in[T]$, $j\in[J]$ and $k\in[5]$.
For notational simplicity define $\beta := \bar\chi(B_j)/(5h^d)$, $\Delta_r:=\hat r_{jk}-\tilde r_{jk}$ and $\Delta_\mu := \hat\mu_{jk}-\tilde \mu_{jk}$.
By Eq.~(\ref{eq:proof-cppt-1}), the condition that $\hat\mu_{jk}\geq c_2=76\varepsilon^{-1}\ln^2(2T^3)$ implies that $2\hat\mu_{jk}\geq \tilde\mu_{jk}\geq \hat\mu_{jk}/2$. Subsequently,
\begin{align}
\left|\frac{\hat r_{jk}}{\hat\mu_{jk}}-\frac{\tilde r_{jk}}{\tilde\mu_{jk}}\right|
&\leq \left|\frac{\tilde r_{jk}+\Delta_r}{\hat \mu_{jk}}-\frac{\tilde r_{jk}}{\tilde\mu_{jk}}\right|\leq \left|\frac{\tilde r_{jk}}{\hat\mu_{jk}}-\frac{\tilde r_{jk}}{\tilde\mu_{jk}}\right| + \frac{2|\Delta_r|}{\tilde\mu_{jk}} = \left|\frac{\tilde r_{jk}}{\tilde\mu_{jk}+\Delta_\mu}-\frac{\tilde r_{jk}}{\tilde\mu_{jk}}\right| + \frac{2|\Delta_r|}{\tilde\mu_{jk}}\nonumber\\
&= \frac{\tilde r_{jk}}{\tilde\mu_{jk}}\left|\frac{\tilde\mu_{jk}}{\tilde\mu_{jk}+\Delta_\mu} - 1\right| + \frac{2|\Delta_r|}{\tilde\mu_{jk}} = \frac{\tilde r_{jk}}{\tilde\mu_{jk}}\frac{|\Delta_\mu|}{\hat\mu_{jk}}+\frac{2|\Delta_r|}{\tilde\mu_{jk}}\leq \frac{2(|\Delta_\mu|+|\Delta_r|)}{\tilde\mu_{jk}},\nonumber
\end{align}
where the last inequality holds because $\tilde r_{jk}/\tilde\mu_{jk}\in[0,1]$ almost surely.
By Eq.~(\ref{eq:proof-cppt-1}) we have that $|\Delta_\mu|+|\Delta_r|\leq 76\varepsilon^{-1}\ln^2(2T^3)$. Subsequently,
\begin{equation}
\left|\frac{\hat r_{jk}}{\hat\mu_{jk}}-\frac{\tilde r_{jk}}{\tilde\mu_{jk}}\right| \leq \frac{152\varepsilon^{-1}\ln^2(2T^3)}{\tilde\mu_{jk}} = \frac{2c_2}{\tilde\mu_{jk}}.
\label{eq:proof-rhojk-1}
\end{equation}
Combining Eq.~(\ref{eq:proof-rhojk-1}) and Eq.~(\ref{eq:cppt-concentration-1}) in Lemma \ref{lem:cppt-concentration}, we have with probability $1-O(T^{-3})$ that
$$
\left|\frac{\hat r_{jk}}{\hat\mu_{jk}} - f_{B_j}(\rho_{jk})\right|\leq \frac{2c_2}{\tilde\mu_{jk}} + \sqrt{\frac{\ln(2T^3)}{2\tilde\mu_{jk}}}.
$$
With a union bound over $t,j$ and $k$ we prove the first inequality in Lemma \ref{lem:rhojk-concentration}.
The second inequality in Lemma \ref{lem:rhojk-concentration} then holds by noting that $\tilde\mu_{jk}\geq \hat\mu_{jk}/2$. $\square$
\end{proof}

Lemma \ref{lem:rhojk-concentration} and Assumption (A2-a) immediately yield the following corollary:
\begin{corollary}
Conditioned on the success event in Lemma \ref{lem:rhojk-concentration}, it holds for all $t\in[T]$ and $j\in[J]$ that $p^*(B_j)\in [\rho_{j1},\rho_{j4}]$,
where $p^*(B_j)=\arg\max_p f_{B_j}(p)$.
\label{cor:cppt-no-elimination}
\end{corollary}

We are now ready to prove Theorem \ref{thm:cppt-regret}.
\begin{proof}{Proof of Theorem \ref{thm:cppt-regret}.}
We will upper bound the regret incurred in each hypercube $B_j$ separately.
The proof is also conditioned on the success events in Lemmas \ref{lem:cppt-aggregation}, \ref{lem:cppt-concentration}, \ref{lem:rhojk-concentration}
and Corollary \ref{cor:cppt-no-elimination}
which occurs with probability $1-O(T^{-1})$.

Fix a particular hypercube $B_j$. We partition the entire $T$ selling periods into \emph{epochs} denoted as $\tau=1,2,3,\cdots$,
with each epoch starting with a time period at which $\varsigma_j$ is reset (at the start of $T$ time periods or as a result of the execution of Line \ref{line:cppt-update1}
or \ref{line:cppt-update2} in Algorithm \ref{alg:cppt}),
and ending when either Line \ref{line:cppt-update1} or Line \ref{line:cppt-update2} is executed again to reset the $\varsigma_j$ pointer.
Let $\mathcal T_j(\tau)$ be the collection of time periods during epoch $\tau$ for hypercube $j$, and define $n_j(\tau) := |\mathcal T_j(\tau)|$.
Note that $\mathcal T_j(\tau)$ includes selling periods during which customers with $\vct x_t\notin B_j$ arrive as well. Moreover, define $\delta_{\tau}:=\rho_{j5}-\rho_{j1}=(\overline p-\underline p)(3/4)^{\tau-1}$ where the equality is by update rule of price range.

Let us fix a particular epoch $\tau$ and let $\mathcal T_j(\tau)$ be the set of all time periods in epoch $\tau$. Recall the definition that $\tilde\mu_{jk}$ is the number of time periods in this epoch
during which price $\rho_{jk}$ is offered to an incoming customer with feature vector $\vct x\in B_j$.
By Assumption (A3-a) and Corollary \ref{cor:cppt-no-elimination}, the total regret incurred in the particular epoch and hypercube compared against $f_{B_j}(p^*(B_j))$ is upper bounded by
\begin{align}
\sum_{t\in\mathcal T_j(\tau)}\vct 1\{\vct x_t\in B_j\}[f_{B_j}(p^*(B_j))-f_{B_j}(p_t)] &\leq 
5\max_k\{\tilde\mu_{jk}\}\times \frac{C_H^2}{2}\max_k\{|\rho_{jk}-p^*(B_j)|^2\}\nonumber\\
&\leq 2.5C_H^2\delta_\tau^2\times \max_k\{\tilde\mu_k\}.
\label{eq:proof-cppt-2}
\end{align}

To upper bound Eq.~(\ref{eq:proof-cppt-2}) we need to upper bound $\tilde\mu_k$. By the description of the CPPQ policy in Algorithm \ref{alg:cppt},
the epoch $\tau$ will terminate when either one of the two following conditions are met:
\begin{eqnarray}
\underline\mu_{1\to 3} \geq c_2 & \qquad\text{and}\qquad & \min\left\{\frac{\hat r_{j2}}{\hat\mu_{j2}}-\frac{\hat r_{j1}}{\hat\mu_{j1}}, \frac{\hat r_{j3}}{\hat\mu_{j3}}-\frac{\hat r_{j2}}{\hat\mu_{j2}}\right\}\geq \frac{3c_1}{\sqrt{\underline\mu_{1\to 3}}} + \frac{3c_1'}{\underline\mu_{1\to 3}},\label{eq:proof-cppt-cond1}\\
\underline\mu_{3\to 5} \geq c_2 & \qquad\text{and}\qquad & \min\left\{\frac{\hat r_{j3}}{\hat\mu_{j3}}-\frac{\hat r_{j4}}{\hat\mu_{j4}},\frac{\hat r_{j4}}{\hat\mu_{j4}}-\frac{\hat r_{j5}}{\hat\mu_{j5}}\right\}\geq \frac{3c_1}{\sqrt{\underline\mu_{3\to 5}}} + \frac{3c_1'}{\underline\mu_{3\to 5}}.\label{eq:proof-cppt-cond2}
\end{eqnarray}
Note that when $p^*(B_j)\in [\rho_{j1},\rho_{j5}]$, Assumption (A2-a) implies that either $\min\{f_{B_j}(\rho_{j2})-f_{B_j}(\rho_{j1}),f_{B_j}(\rho_{j3})-f_{B_j}(\rho_{j2})\}\geq \frac{\sigma_H^2}{32}\delta_\tau^2$
or $\min\{f_{B_j}(\rho_{j3})-f_{B_j}(\rho_{j4}),f_{B_j}(\rho_{j4})-f_{B_j}(\rho_{j5})\}\geq \frac{\sigma_H^2}{32}\delta_\tau^2$ holds.
Assume without loss of generality that $\min\{f_{B_j}(\rho_{j2})-f_{B_j}(\rho_{j1}),f_{B_j}(\rho_{j3})-f_{B_j}(\rho_{j2})\}\geq \frac{\sigma_H^2}{32}\delta_\tau^2$.
Then by Lemma \ref{lem:rhojk-concentration}, the condition (\ref{eq:proof-cppt-cond1}) is satisfied when
\begin{align}
\underline\mu_{1\to 3}&= c_2 +\const\times\left[ \frac{c_1^2}{\sigma_H^4\delta_\tau^4} + \frac{c_1'}{\sigma_H^2\delta_\tau^2}\right].
\label{eq:proof-cppt-3}
\end{align}
Note that $\hat\mu_{jk}\geq\tilde\mu_{jk}/2$ thanks to Lemma \ref{lem:cppt-concentration}. Eq.~(\ref{eq:proof-cppt-3}) and the symmetric case of $\min\{f_{B_j}(\rho_{j3})-f_{B_j}(\rho_{j4}),f_{B_j}(\rho_{j4})-f_{B_j}(\rho_{j5})\}\geq \frac{\sigma_H^2}{32}\delta_\tau^2$ are then implied by
\begin{equation}
\min_k\{\tilde\mu_{jk}\} \leq 2c_2 + \const\times\left[\frac{c_1^2}{\sigma_H^4\delta_\tau^4} + \frac{c_1'}{\sigma_H^2\delta_\tau^2}\right].
\label{eq:proof-cppt-4}
\end{equation}

On the other hand, note that $|f_{B_j}(\rho_{j,k+1})-f_{B_j}(\rho_{j,k})|\leq \frac{C_H^2}{2}\delta_\tau^2$ for all $k\in\{1,2,3,4,5\}$.
With Lemma \ref{lem:rhojk-concentration} and the stopping condition in Eqs.~(\ref{eq:proof-cppt-cond1},\ref{eq:proof-cppt-cond2}),
we have at the end of epoch $\tau$ that
\begin{equation}
\max_k\{\tilde\mu_{jk}\}\geq \max_k\{\hat\mu_{jk}/2\} \geq \frac{c_2}{2} + \frac{2c_1^2}{C_H^2\delta_\tau^4} + \frac{c_1'}{C_H\delta_\tau^2}.
\label{eq:proof-cppt-4half}
\end{equation}

Contrasting Eq.~(\ref{eq:proof-cppt-4}) with Eq.~(\ref{eq:proof-cppt-2}), we need to upper bound the differences between $\tilde\mu_{jk}$ for $k\in[5]$.
This can be done by using Eq.~(\ref{eq:cppt-concentration-1}) in Lemma \ref{lem:cppt-concentration} and the triangle inequality, which yield
for every $k,k'\in\{1,2,3,4,5\}$ that
\begin{align}
\big|\tilde\mu_{jk}-\tilde\mu_{jk'}\big|&\leq 2\sqrt{C_Xh^d n_j\ln(2T^3)} + 3\ln(2T^3).
\label{eq:proof-cppt-5}
\end{align}

Combine Eqs.~(\ref{eq:proof-cppt-2},\ref{eq:proof-cppt-4},\ref{eq:proof-cppt-4half},\ref{eq:proof-cppt-5}). The total regret incurred in hypercube $B_j$ and epoch $\tau$
(such that $\delta_\tau=\rho_{j4}-\rho_{j1}$) is upper bounded by
\begin{align}
&2C_H^2\delta_\tau^2\times \const\times \left[ c_2 + \frac{c_1^2}{\sigma_H^4\delta_\tau^4} + \frac{c_1'}{\sigma_H^2\delta_\tau^2} + \sqrt{C_Xh^d n_j\ln(2T^3)} + \ln(2T^3) \right]\nonumber\\
&\leq \frac{\const\times c_1^2C_H^2}{\sigma_H^4\delta_\tau^2} + \const\times C_H^2\left[c_2+\sigma_H^{-2}c_1'+\sqrt{C_Xh^d n_j\ln(2T^3)} + \ln(2T^3)\right]\nonumber\\
&\leq \frac{\const\times c_1^2C_H^2}{\sigma_H^4}\sqrt{\frac{C_H^2\max_k\tilde\mu_{jk}}{2c_1^2}} + \const\times C_H^2\left[c_2+\sigma_H^{-2}c_1'+\sqrt{C_Xh^d T\ln(2T^3)} + \ln(2T^3)\right]
\label{eq:proof-cppt-6}\\
&\leq \const\times c_1C_H^3\sigma_H^{-4}\max_k\{\sqrt{\tilde\mu_{jk}}\} + \const\times C_H^2\left[c_2+\sigma_H^{-2}c_1'+\sqrt{C_Xh^d T\ln(2T^3)} + \ln(2T^3)\right].\label{eq:proof-cppt-7}
\end{align}
Here, Eq.~(\ref{eq:proof-cppt-6}) holds because $\delta_\tau^{-4}\leq C_H^2\max_k\{\tilde\mu_{jk}\}/(2c_1^2)$ thanks to Eq.~(\ref{eq:proof-cppt-4half}).
Note also that Eq.~(\ref{eq:proof-cppt-4half}) implies an upper bound of $\ln(C_H^2 T)$ on the total number of epochs for each $j$.
Define $C_1' := \const\times c_1C_H^3\sigma_H^{-4}\ln(C_H^2 T)\leq \const \times C_H^3\sigma_H^{-4}\ln^2(2C_H^2T^3)$, $C_2' := \const\times C_H^2\varepsilon(c_2+\sigma_H^{-2}c_1')\ln(C_H^2 T)\leq\const\times C_H^2\sigma_H^{-2}\ln^3(2C_H^2T^3)$ and $C_3'=\const\times C_H^2(\sqrt{C_X\ln(2T^3)}+\ln(2T^3))\ln(C_H^2 T)\leq \const\times C_H^2\sqrt{C_X}\ln^2(2C_H^2 T^3)$.
Summing Eq.~(\ref{eq:proof-cppt-7}) over $j\in[J]$ and $k\in[5]$, we have that
\begin{align}
\sum_{j=1}^J\sum_{t=1}^T& \vct 1\{\vct x_t\in B_j\}[f_{B_j}(p^*(B_j))-f_{B_j}(p_t)] \leq C_1'\sum_{j,k}\sqrt{\tilde\mu_{jk}} + C_2' J/\varepsilon + C_3'J\times (\sqrt{h^d T}+1)\nonumber\\
&\leq C_1'\sqrt{5J}\sqrt{\sum_{j,k}\tilde\mu_{jk}} + \frac{C_2' J}{\varepsilon} + 2C_3' J\sqrt{h^d T}\leq \frac{2(C_1'+C_3')\sqrt{T}}{h^{d/2}} + \frac{C_2'}{\varepsilon h^d}. \label{eq:proof-cppt-8}
\end{align}

Additionally, Assumption (A3-a) with $B=\{\vct x_t\}$ and (A3-b), (A3-c) imply that $|f(p^*(\vct x),\vct x)-f(p^*(B_j),\vct x)|\leq \frac{C_H^2C_p^2}{2}\sup_{\vct x,\vct x'\in B_j}\|\vct x-\vct x'\|_2^2\leq \frac{C_H^2C_p^2d}{2} h^2$. This together with Eq.~(\ref{eq:proof-cppt-8}) yields with probability $1-O(T^{-1})$ that
\begin{align}
\sum_{t=1}^T f(p^*(\vct x_t),\vct x_t)-f(p_t,\vct x_t)
&\leq  \frac{2(C_1'+C_3')\sqrt{T}}{h^{d/2}} + \frac{C_2'}{\varepsilon h^d} + \frac{C_H^2C_p^2 d}{2} h^2 T.
\end{align}
With the choice $J=\lceil{T}^{d/(d+4)}\rceil$ corresponding to $h=J^{-1/d}\approx T^{-1/(d+4)}$, 
the above inequality proves Theorem \ref{thm:cppt-regret} with $2(C_1'+C_3')+C_H^2C_p^2d/2 \leq \const\times C_H^2(\sigma_H^{-4}+C_H\sqrt{C_X})\ln^2(2C_H^2T^3)+C_H^2C_p^2d/2$ and $\bar C_1'=2C_2' \leq \const\times C_H^2\sigma_H^{-2}\ln^3(2C_H^2T^3)$. $\square$
\end{proof}

\section{Proofs of technical lemmas in Section~\ref{sec:lower-bound}}

\subsection{Proof of Lemma \ref{lem:reduction-classification}}

Let $\eta:= (1-2^{-1/d})/2$ and define $H_j := \{\vct x\in B_j: \sd(\vct x,\partial B_j)\geq \eta h\}$.
Since $P_X$ is the uniform distribution on $[0,1]^d$, we have $\Pr[\vct x\in H_j]= 0.5\Pr[\vct x\in B_j] = 0.5 h^d$.
For any $\vct x\in H_j$, by simple algebra we have that
\begin{align}
p^*(\vct x) &= \frac{2}{3}, & \text{if }\nu_j=0;\label{eq:proof-reduction-1}\\
p^*(\vct x) &= \frac{2}{3} - \frac{\sd(\vct x,\partial B_j)}{3(1+\sd(\vct x,\partial B_j))}\leq \frac{2}{3} - \frac{1}{6}\eta h,& \text{if }\nu_j=1.\label{eq:proof-reduction-2}
\end{align}
Define $\mathcal A := \{p: p\geq \frac{2}{3} - \frac{\eta h}{12}\}$ and $\mathcal B := \{p:p\leq \frac{2}{3}-\frac{\eta h}{12}\}$.
Since $f_{\vct\nu}(p,\vct x)$ is strongly concave in $p$, we have that
\begin{align}
f_{\vct\nu}(p,\vct x) &\geq \frac{\eta^2 h^2}{144},& \forall p\in\mathcal B,\;\;\text{if }\nu_j=0;\label{eq:proof-reduction-3}\\
f_{\vct\nu}(p,\vct x) &\geq \frac{\eta^2 h^2}{144},& \forall p\in\mathcal A,\;\;\text{if }\nu_j=1.\label{eq:proof-reduction-4}
\end{align}

We next augment the $\varepsilon$-LDP personalized pricing policy $\pi$ to construct a $2\varepsilon$-LDP policy $\pi'$.
Suppose $z_1,\cdots,z_T$ are the intermediate quantities produced policy $\pi$, such that the distribution of $z_t$ is measurable conditioned on
$s_t$ and $z_1,\cdots,z_{t-1}$.
Construct augmented intermediate quantity $z_t' = (z_t,\vct\alpha_t,\vct\beta_t)$, where $\vct\alpha_t,\vct\beta_t\in\mathbb R^J$ are defined as
\begin{align*}
\alpha_{tj} &= \vct 1\{\vct x_t\in H_j\wedge p_t\in\mathcal A\} + w_{tj}, & w_{tj}\overset{i.i.d.}{\sim}\Lap(2/\varepsilon);\\ 
\beta_{tj} &= \vct 1\{\vct x_t\in H_j\wedge p_t\in\mathcal B\} + w_{tj}', & w_{tj}'\overset{i.i.d.}{\sim}\Lap(2/\varepsilon).
\end{align*}
It is easy to verify that the distribution of $(\vct\alpha_t,\vct\beta_t)$ is measurable conditioned on $s_t=(\vct x_t,y_t,p_t)$,
and furthermore $z_t'$ satisfies $2\varepsilon$-LDP thanks to the Laplace mechanism and simple composition of two $\varepsilon$-LDP procedures.
This shows that $\pi'$ satisfies $2\varepsilon$-LDP.
Furthermore, the $\{A_t\}_{t=1}^T$ functions in $\pi'$ are the same as those in $\pi$, meaning that $\pi'$ has exactly the same regret performance as $\pi$.

Next, we construct the classifier $\psi:(z_1',\cdots,z_T')\mapsto\hat{\vct\nu}\in\{0,1\}^J$.
In fact, we will only use $\{\vct\alpha_t,\vct\beta_t\}_{t=1}^T$ in $(z_1',\cdots,z_T')$.
For each hypercube $B_j$, define $\hat\alpha_j := \sum_{t=1}^T\alpha_{tj}$ and $\hat\beta_j := \sum_{t=1}^T\beta_{tj}$.
The classifier is then defined as
\begin{equation}
\hat\nu_j := \left\{\begin{array}{ll} 1,& \text{if }\hat\alpha_j\leq\hat\beta_j;\\
0,& \text{if }\hat\alpha_j>\hat\beta_j,\end{array}\right.\;\;\;\;\;\;\forall j\in[J].
\label{eq:defn-psi}
\end{equation}

We will now analyze the performance of $\hat{\vct\nu}$ constructed in Eq.~(\ref{eq:defn-psi}).
Define $\alpha_j := \sum_{t=1}^T\vct 1\{\vct x_t\in H_j\}\vct 1\{p_t\in\mathcal A\}$ and $\beta_j :=\sum_{t=1}^T\vct 1\{\vct x_t\in H_j\}\vct 1\{p_t\in\mathcal B\}$.
Invoke concentration inequalities for i.i.d.~Laplace random variables \citep[Lemma 2.8]{chan2011private}.
If $T\geq 6.2\sqrt{\ln (2T)}$ then with probability $1-O(T^{-2})$ uniformly over $j\in[J]$, it holds that
\begin{equation}
\max\big\{\big|\hat\alpha_j-\alpha_j\big|, \big|\hat\beta_j-\beta_j\big|\big\}\leq 14\varepsilon^{-1}\sqrt{T\ln T}.
\label{eq:diff-alpha-beta}
\end{equation}
Suppose $\hat\nu_j\neq\nu_j$, and without loss of generality assume $\nu_j=0$ and $\hat\nu_j=1$.
This means $\hat\alpha_j\leq\hat\beta_j$, 
which implies (by Eq.~(\ref{eq:diff-alpha-beta})) with probability $1-O(T^{-2})$ that
\begin{equation}
\sum_{t=1}^T\vct 1\{\vct x_t\in H_j\wedge p_t\in\mathcal B\} \geq \frac{1}{2}\left(\sum_{t=1}^T\vct 1\{\vct x_t\in H_j\}\right) - 14\varepsilon^{-1}\sqrt{T\ln (2T)}.
\label{eq:proof-reduction-1}
\end{equation}
Note that $\Pr[\vct x_t\in H_j]=1/(2J)$ because $P_X$ is the uniform distribution on $[0,1]^d$. By Hoeffding's inequality, with probability $1-O(T^{-2})$ we have
\begin{align}
\sum_{t=1}^T\vct 1\{\vct x_t\in H_j\wedge p_t\in\mathcal B\}
&\geq \frac{T}{4J} - 2.5\sqrt{T\ln(2T)} - 14\varepsilon^{-1}\sqrt{T\ln (2T)}\geq \frac{T}{4J}-16.5\varepsilon^{-1}\sqrt{T\ln(2T)}\nonumber\\
&\geq \frac{T}{8J},\label{eq:proof-reduction-2}
\end{align}
where the last inequality holds by the condition in Lemma \ref{lem:reduction-classification} that $\sqrt{T}\geq 66\varepsilon^{-1}J\sqrt{\ln(2T)}$.
Eqs.~(\ref{eq:proof-reduction-3},\ref{eq:proof-reduction-4}) then imply that with probability $1-O(T^{-2})$,
$$
\sum_{t=1}^T\vct 1\{\vct x_t\in H_j\}(f(p^*(\vct x_t),\vct x_t)-f(p_t,\vct x_t))\geq \frac{T}{8J}\cdot \frac{\eta^2h^2}{144}= \frac{\eta^2 h^2T}{1152J}.
$$
Taking expectations on both sides of the above inequality and summing over all $j\in[J]$, we obtain
\begin{align*}
\mathbb E\left[\sum_{t=1}^T f(p^*(\vct x_t),\vct x_T)-f(p_t,\vct x_t)\right]
\geq \frac{\eta^2h^2T}{1152J}\mathbb E\left[\sum_{j=1}^J\vct 1\{\hat\nu_j\neq\nu_j\}\right].
\end{align*}
Take a supreme over all $\vct\nu\in\{0,1\}^J$ on both sides of the above inequality and note that $\eta=(1-2^{-1/d})/2\geq (1-e^{-1/(2d)})/2 \geq (1-(1-1/(4d)))/2 = 1/(8d)$. We complete the proof of Lemma \ref{lem:reduction-classification}.

\subsection{Proof of Lemma \ref{lem:sddi-ldp}}

Fix $j\in[J]$. Let $M_{\pm j,t}^\pi(\cdot)$ be the distributions of $z_t$, which is measurable conditioned on $\vct z_{<t}=(z_1,\cdots,z_{t-1})$
thanks to the definition in Eqs.~(\ref{eq:defn-Qt},\ref{eq:defn-At}). 
More specifically, $M_{\pm j,t}^\pi(Z|\vct z_{<t})=\int_{\mathcal S}Q_t(Z|s_t,\vct z_{<t})\ud P_{\pm j,t}^\pi(s_t|\vct z_{<t})$. By the chain rule of KL-divergence, we have that
\begin{equation}
D_{\kl}^\sy(M_{+j}^\pi,M_{-j}^\pi) \leq 2\sum_{t=1}^T\int_{\mathcal Z^{t-1}}D_\kl^\sy(M_{+j,t}^\pi(\cdot|\vct z_{<t}),M_{-j,t}^\pi(\cdot|\vct z_{<t}))\ud\bar M_{<t}^\pi(\ud z_{<t}).
\label{eq:proof-sddi-1}
\end{equation}

Now fix $t\in[T]$ and define $\varsigma_{jt}(\vct z_{<t}) := D_\kl^\sy(M_{+j,t}^\pi(\cdot|\vct z_{<t}),M_{-j,t}^\pi(\cdot|\vct z_{<t}))$.
Let $m_{+j}(\cdot|\vct z_{<t}),m_{-j}(\cdot|\vct z_{<t})$ be the PDFs of $M_{+j,t}^\pi(\cdot|\vct z_{<t})$ and $M_{-j,t}^\pi(\cdot|\vct z_{<t})$.
Define also $m^0(z|\vct z_{<t}) := \inf_{s_t\in\mathcal S}q_t(z|s_t,\vct z_{<t})$, where $q_t$ is the PDF of $Q_t$ defined in Eq.~(\ref{eq:defn-Qt}).
We then have
\begin{align}
\varsigma_{jt}(\vct z_{<t}) 
&= \int_{\mathcal Z}[m_+(z|\vct z_{<t})-m_-(z|\vct z_{<t})]\ln\frac{m_+(z|\vct z_{<t})}{m_-(z|\vct z_{<t})}\ud z\nonumber\\
&\leq \int_{\mathcal Z}[m_+(z|\vct z_{<t})-m_-(z|\vct z_{<t})]^2\cdot \frac{\ud z}{\min\{m_+(z|\vct z_{<t}),m_-(z|\vct z_{<t})\}}\nonumber\\
&\leq \int_{\mathcal Z}[m_+(z|\vct z_{<t})-m_-(z|\vct z_{<t})]^2\cdot \frac{\ud z}{m^0(z|\vct z_{<t})},\label{eq:proof-sddi-2}
\end{align}
where the first inequality holds because $|\ln(a/b)|\leq |a-b|/\min\{a,b\}$ for all $a,b\in\mathbb R_+$ (see, e.g., \citealt[Lemma 4]{duchi2018minimax}).
Note that $m_{\pm}(z|\vct z_{<t}) = \int_{\mathcal S}q_t(z|s_t,\vct z_{<t}) dP_{\pm j,t}^\pi(s_t|\vct z_{<t})$, and $\int_{\mathcal S}\ud P_{+j,t}^\pi(s_t|\vct z_{<t})-\ud P_{-j,t}^\pi(s_t|\vct z_{<t})=1-1=0$. We then have
\begin{align*}
m_+(z|\vct z_{<t})-m_-(z|\vct z_{<t}) &= \int_{\mathcal S}q_t(z|s_t,\vct z_{<t})\big[\ud P_{+j,t}^\pi(s_t|\vct z_{<t})-\ud P_{-j,t}^\pi(s_t|\vct z_{<t})\big]\\
&= \int_{\mathcal S}(q_t(z|s_t,\vct z_{<t})-m^0(z|\vct z_{<t}))\big[\ud P_{+j,t}^\pi(s_t|\vct z_{<t})-\ud P_{-j,t}^\pi(s_t|\vct z_{<t})\big]\\
&= m^0(z|\vct z_{<t})\int_{\mathcal S}\left(\frac{q_t(z|s_t,\vct z_{<t})}{m^0(z|\vct z_{<t})}-1\right)\left[\ud P_{+j,t}^\pi(s_t|\vct z_{<t})-\ud P_{-j,t}^\pi(s_t|\vct z_{<t})\right].
\end{align*}
Subsequently,
Eq.~(\ref{eq:proof-sddi-2}) can be upper bounded by
$$
\varsigma_{jt}(\vct z_{<t}) \leq \int_{\mathcal Z}m^0(z|\vct z_t)\left|\int_{\mathcal S}\left(\frac{q_t(z|s_t,\vct z_{<t})}{m^0(z|\vct z_{<t})}-1\right)\left[\ud P_{+j,t}^\pi(s_t|\vct z_{<t})-\ud P_{-j,t}^\pi(s_t|\vct z_{<t})\right]\right|^2\ud z
$$

Note that $q_t(z|s_t,\vct z_{<t})/m^0(z|\vct z_{<t})\leq e^{2\varepsilon}$ for all $z\in\mathcal Z$, because $\pi$ satisfies $2\varepsilon$-LDP.
Note also that $\int_{\mathcal Z}m^0(z|\vct z_{<t})\ud z=\int_{\mathcal Z}\inf_{s_t\in\mathcal S}q_t(z|s_t,\vct z_{<t})\ud z\leq \inf_{s_t\in\mathcal S}\int_{\mathcal Z}q_t(z|s_t,\vct z_{<t})\ud z\leq 1$. The above inequality can then be reduced to
\begin{align}
\varsigma_{jt}(\vct z_{<t}) \leq (e^{2\varepsilon}-1)^2\sup_{\|\gamma\|_{\infty}\leq 1}\left|\int_{\mathcal S}\gamma(s_t,\vct z_{<t})\left[\ud P_{+j,t}^\pi(s_t|\vct z_{<t})-\ud P_{-j,t}^\pi(s_t|\vct z_{<t})\right]\right|^2.
\label{eq:proof-sddi-3}
\end{align}
Combining Eqs.~(\ref{eq:proof-sddi-1},\ref{eq:proof-sddi-3}) we obtain
\[
\begin{split}
&D_{\kl}^\sy(M_{+j}^\pi,M_{-j}^\pi) \\
\leq& 2(e^{2\varepsilon}-1)^2\sum_{t=1}^T\sup_{\|\gamma\|_{\infty}\leq 1}\int_{\mathcal Z^{t-1}}\left|\int_{\mathcal S}\gamma(s_t,\vct z_{<t})\left[\ud P_{+j,t}^\pi(s_t|\vct z_{<t})-\ud P_{-j,t}^\pi(s_t|\vct z_{<t})\right]\right|^2\ud\bar M_{<t}^\pi(\ud z_{<t}).
\end{split}
\]
Summing both sides of the above inequality over $j=1,2,\cdots,J$ we complete the proof of Lemma \ref{lem:sddi-ldp}.

\subsection{Proof of Lemma \ref{lem:ppm-ldp}}

Recall the decomposition that $P_{\pm j,t}^\pi(\vct x_t,y_t,p_t|\vct z_{<t})=\chi(\vct x_t)A_t(p_t|\vct x_t,\vct z_{<t})\Pr_{\pm j}(y_t|p_t,\vct x_t)$.
Because $P_{+j}^\pi$ and $P_{-j}^\pi$ are exactly the same for $\vct x\notin B_j$, it suffices to consider only $\vct x\in B_j$.
For any $\vct x\in B_j$, the distribution of $\vct x$ and $p$ is independent of the underlying demand model $f_{\vct\nu}$. Hence, 
\begin{align*}
\|P_{+j,t}^{\pi}(\cdot|\vct z_{<t})-P_{-j,t}^\pi(\cdot|\vct z_{<t})\|_{\tv}
&\leq \frac{1}{2}\int_{ B_j}\mathbb E_{p\sim A_t(\cdot|\vct x,\vct z_{<t})}\left[\sup_{\vct\nu_j\neq\vct\nu'_j}\big|\lambda_{\vct\nu}(p,\vct x)-\lambda_{\vct\nu'}(p,\vct x)\big|\right]\ud P_X(\vct x)\\
&\leq \frac{1}{2}\int_{ B_j}\mathbb E_{p\sim A_t(\cdot|\vct x,\vct z_{<t})}\left[\left|\frac{1}{3}-\frac{p}{2}\right|\sup_{\vct x\in B_j}\sd(\vct x,\partial B_j)\right]\ud P_X(\vct x)\\
&\leq \frac{\sqrt{d}h}{4}\int_{ B_j}\mathbb E_{p\sim A_t(\cdot|\vct x,\vct z_{<t})}[|p-p_0|]\ud P_X(\vct x),
\end{align*}
where $p_0=2/3$.
Note that $\int_{ B_j}\ud P_X(\vct x)=\Pr[\vct x\in B_j]=1/J$.
Subsequently, by Jensen's inequality that $(\mathbb E[\cdot])^2\leq \mathbb E[\cdot^2]$, we have
$$
\|P_{+j,t}^{\pi}(\cdot|\vct z_{<t})-P_{-j,t}^\pi(\cdot|\vct z_{<t})\|_{\tv}^2 \leq \frac{dh^2}{16J^2}\mathbb E_{\vct x\sim U(B_j)}\mathbb E_{p\sim A_t(\cdot|\vct x,\vct z_{<t})}[(p-p_0)^2],
$$
where $U(B_j)$ is the uniform distribution on $B_j$. This proves Lemma \ref{lem:ppm-ldp}.

\subsection{Proof of Lemma \ref{lem:regret-self-bound}}

Let $\lambda_{\vct\nu},f_{\vct\nu}$ be the demand model corresponding to an arbitrary $\vct\nu\in\{0,1\}^J$,
and $\pi$ be a personalized pricing policy.
Let $\vct x\in B_j$ be an arbitrary context vector belonging to hypercube $B_j$, and $p^*(\vct x)=\arg\max_{p}f_{\vct\nu}(p,\vct x)$.
Eqs.~(\ref{eq:proof-reduction-1},\ref{eq:proof-reduction-2}) in the proof of Lemma \ref{lem:reduction-classification}
show that $|p^*(\vct x)-p_0|\leq \eta h/6\leq h/12$, where $p_0=2/3$, regardless of the value of $\nu_j$.
Note also that $f_{\vct\nu}(p,\vct x)$ is strongly concave with respect to $p$. Subsequently, we have
for all $\vct x\in[0,1]^d$ and $p\in[0,1]$ that
\begin{align}
f(\vct p^*(\vct x),\vct x) - f(p,\vct x)\geq (p-p^*(\vct x))^2 \geq \inf_{|p'-p_0|\leq h/12}(p-p')^2 \geq \frac{1}{4}\phi(|p-p_0|^2,h^2).
\end{align}
Consequently, 
\begin{align*}
\sR(\pi) &= \sup_{\vct\nu\in\{0,1\}^J}\mathbb E_{\vct\nu}^\pi\left[\sum_{t=1}^T f_{\vct\nu}(p^*(\vct x_t),\vct x_t)-f_{\vct\nu}(p_t,\vct x_t)\right]
\geq \frac{1}{2^J}\sum_{\vct\nu\in\{0,1\}^J}\mathbb E_{\vct\nu}^\pi\left[\sum_{t=1}^T f_{\vct\nu}(p^*(\vct x_t),\vct x_t)-f_{\vct\nu}(p_t,\vct x_t)\right]\\
&\geq  \frac{1}{2^J}\sum_{\vct\nu\in\{0,1\}^J}\mathbb E_{\vct\nu}^\pi\left[\sum_{t=1}^T\frac{1}{4}\phi(|p_t-p_0|^2,h^2)\right]\\
&= \frac{1}{2^J}\sum_{\vct\nu\in\{0,1\}^J}\sum_{t=1}^T\mathbb E\left[\frac{1}{4}\phi(|p_t-p_0|^2,h^2)\bigg| p_t\sim A_t(\cdot|\vct x_t,\vct z_{<t}),\vct x_t\sim P_X,\vct z_{<t}\sim M_{\vct\nu,<t}^\pi\right]\\
&= \frac{1}{4}\sum_{t=1}^T\mathbb E\left[\phi(|p_t-p_0|^2,h^2)\bigg| p_t\sim A_t(\cdot|\vct x_t,\vct z_{<t}),\vct x_t\sim P_X,\vct z_{<t}\sim 
\bar M_{<t}^\pi\right],
\end{align*}
which is to be proved. 

\end{document}